\documentclass{article}

\usepackage[utf8]{inputenc} 
\usepackage[T1]{fontenc}    
\usepackage{hyperref}       
\usepackage{url}            
\usepackage{booktabs}       
\usepackage{amsfonts}       
\usepackage{nicefrac}       

\usepackage{microtype}
\usepackage{graphicx}
\usepackage{subfigure}
\usepackage{booktabs} 

\usepackage{cite}
\usepackage{bbm}
\usepackage{url}
\usepackage{color}
\usepackage{amsthm}
\usepackage{booktabs}
\usepackage{amsmath, amssymb}
\usepackage{mathtools}
\usepackage{amsfonts}
\usepackage{graphicx}
\usepackage{algorithm}
\usepackage[noend]{algorithmic}
\usepackage{subfigure}
\usepackage{amssymb, amsmath,graphicx,charter, latexsym}
\usepackage{enumerate}

\usepackage{geometry}
\usepackage{algorithm}
\usepackage[noend]{algorithmic}

\newtheorem{definition}{Definition}

\newtheorem{lemma}{Lemma}

\newtheorem{corollary}{Corollary}
\newtheorem{theorem}{Theorem}

\newtheorem{assumption}{Assumption}

\DeclarePairedDelimiter{\norm}{\lVert}{\rVert}
\DeclarePairedDelimiter{\ceil}{\lceil}{\rceil}

\usepackage{hyperref}
\usepackage[switch]{lineno}

\title{Learning with Safety Constraints: \\Sample Complexity of Reinforcement Learning for Constrained MDPs}

\author{Aria HasanzadeZonuzy, Archana Bura,
Dileep Kalathil and
Srinivas Shakkottai \thanks{Authors are with the Department of Electrical and Computer Engineering, Texas A\& M University, Texas, USA. {\small Email:\{azonuzy, dileep.kalathil, sshakkot\}@tamu.edu}}}

\date{}

\allowdisplaybreaks

\begin{document}

\maketitle

\begin{abstract}

Many physical systems have underlying safety considerations that require that the policy employed ensures the satisfaction of a set of constraints.  The analytical formulation usually takes the form of a Constrained Markov Decision Process (CMDP).
We focus on the case where the CMDP is unknown, and RL algorithms obtain samples to discover the model and compute an optimal constrained policy.  Our goal is to characterize the relationship between safety constraints and the number of samples needed to ensure a desired level of accuracy---both objective maximization and constraint satisfaction---in a PAC sense.  We explore two classes of RL algorithms, namely, (i) a generative model based approach, wherein samples are taken initially to estimate a model, and (ii) an online approach, wherein the model is updated as samples are obtained.  Our main finding is that compared to the best known bounds of the unconstrained regime, the sample complexity of constrained RL algorithms are increased by a factor that is logarithmic in the number of constraints, which suggests that the approach may be easily utilized in real systems. 

\end{abstract}

\section{Introduction}
\label{sec:intro}

Markov Decision Processes (MDPs) are used to model a variety of systems for which stationary control policies are appropriate.  In many cyber-physical systems (algorithmically controlled physical systems) restrictions may be placed on functions of the probability with which states may be visited.  For example, in power systems, the frequency must be kept within tolerable limits, and allowing it to go outside these tolerances often might be unsafe.  Similarly, in communication systems the number of transmissions that may be made in a time interval is limited by an average radiated power constraint due to interference and human safety considerations.  The number of constraints can be large, since they can represent physical limitations (e.g., communication or transmission link capacities), performance requirements (per-flow packet delays, tolerable frequencies) and so on.  The Constrained-MDP (CMDP) framework is used to model such circumstances~\cite{altman}.      

In this paper, our objective is to design simple algorithms to solve CMDP problems under an unknown model. Whereas the goal of a typical model-based RL approach would take as few samples as possible to quickly determine the optimal policy, minimizing the number of samples taken is even more important in the CMDP setting.  This because constraints are violated during the learning process, and it might be critical to keep the number of such violations as low as possible due to safety considerations mentioned earlier, and yet ensure that the system objectives are maximized.  Hence, determining how the joint metrics of objective maximization and safety violation evolve over time as the model becomes more and more accurate is crucial to understand the efficacy of a proposed RL algorithm for CMDPs.

\paragraph{Main Contributions:}  Our goal is to analyze the sample complexity of solving CMDPs to a desired accuracy with a high probability in both objective and constraints in the context of finite horizon (episodic) problems.  We focus on two figures of merit pertaining to objective maximization and constraint satisfaction in a probably-approximately-correct (PAC) sense.  
Our main contributions are as follows:\vspace{0.05in}\\
(i) We develop two model-based algorithms, namely, (i) a generative approach that obtains samples initially then creates a model, and (ii) an online approach in which the model is updated as time proceeds.  In both cases, the estimated model might have no solution, and we utilize a confidence-ball around the estimate to ensure that a solution may be found with high probability (assuming that the real model has a solution).\vspace{0.05in}\\
(ii) The algorithms follow the general pattern of model construction or update, followed by a solution using linear programming (LP) of the CMDP generated in this manner, with the addendum that the LP is extended to account for the fact that a search is made over the entire ball of models given the current samples. This procedure not only contributes to optimism as \cite{efroni}, but also guarantees feasibility of the solution. \vspace{0.05in}\\
(iii) We develop PAC-type sample complexity bounds for both algorithms, accounting for both objective maximization and constraint satisfaction.  
The general intuition is that the model accuracy should be higher than in the unconstrained case and, our main finding agrees with this intuition. Furthermore, comparing our main results with lower bounds on sample complexity of MDPs~\cite{azar,ucfh}, we discover that the increase in the sample complexity is by a logarithmic factor in the number of constraints and a size of state space. However, there are no lower bound results for CMDPs to the best of our knowledge. 

As mentioned above, the number of constraints in cyber-physical systems can be large.   Our result indicating logarithmic scaling with the number of constraints indicates that the number of constraints is not a major concern in solving unknown CMDPs via RL, hence indicating that the practicality of applying the constrained RL approach to cyber-physical systems applications. 


\paragraph{Related Work:}

Much work in the space of CMDP has been driven by problems of control, and many of the algorithmic approaches and applications have taken a control-theoretic view~\cite{altman,altman2002applications,borkar2005actor,borkar2014risk,singh2018throughput,singh2014fluctuation}.  The approach taken is to study the problem under a known model, and showing asymptotic convergence of the solution method proposed. There are also studies on constrained partially observable MDPs such as \cite{value-iteration-1, value-iteration-2}. Both of these works propose algorithms based on value iteration requiring solving linear program or constrained quadratic program.

Extending CMDP approaches to the context on an unknown model has also mostly focused on asymptotic convergence \cite{bhatnagar2012online,chow2018lyapunov,tessler2018reward,paternain2019constrained} under Lagrangian methods to show zero eventual duality gap. \cite{ipo} also proposes an algorithm based on Lagrangian method, but proves that this algorithm achieves a small eventual gap. On the other hand empirical works built on Lagrangian method has also been proposed \cite{modiano}.

A parallel theme has been related to the constrained bandit case, wherein the the underlying problem, while not directly being an MDP, bears a strong relation to it.  Work such as \cite{badanidiyuru2013bandits,wu2015algorithms,amani2019linear} consider such constraints, either in a knapsack sense, or on the type of controls that may be applied in a linear bandit context.

Closest to our theme are parallel works on CMDPs.  For instance, \cite{zheng2020constrained} and \cite{icml2} present results in the context of unknown reward functions, with either a known stochastic or deterministic transition kernel.  Other work~\cite{icml1} focuses on asymptotic convergence, and so does not provide an estimate on the learning rate.  Finally, \cite{efroni} explores algorithms and themes similar to ours, but focuses on characterizing objective and constrained regret under different flavors of online algorithms, which can be seen as complementary to or work.  Since there is no direct relation between regret and sample complexity~\cite{ubev}, applying their regret approach to our setting gives relatively weak sample complexity bounds.  Our discovery of a general principle of logarithmic increase in sample complexity with the number of constraints also distinguishes our work.





\section{Notation and Problem Formulation}
\label{sec:not}

\paragraph{Notation and Setup:} We consider a general finite-horizon CMDP formulation. There are a set of states $S$ and set of actions $A.$ The reward matrix is denoted by $r,$ under which $r(s, a)$ is the reward for any state-action pair $(s, a).$  We assume that there are $N$ constraints. We use $c$ to denote the cost matrix, where $c(i, s, a)$ is the immediate cost incurred by the $i^{th}$ constraint in $(s, a)$ where $i \in \{1, \dots, N \}.$  Also, the vector $\bar{C}$ is used to denote the value of the constraints (i.e., the bound that must be satisfied). The probability of reaching another state $s'$ while being at state $s$ and taking action $a$ is determined by transition kernel $P(s'|s,a).$ At the beginning of each horizon, we begin from a fixed initial state $s_0.$ As the CMDP has a finite horizon, the length of each horizon, or episode, is considered to be a fixed value $H.$  Hence, the CMDP is defined by the tuple $M = \langle S, A, P, r, c, \bar{C}, s_0, H \rangle .$ 

\begin{assumption}
\label{a:cmdp}
We assume $S$ and $A$ are finite sets with cardinalities $|S|$ and $|A|.$  Further, we assume that the immediate reward $r(s,a)$ is taken from the interval $[0, 1]$ and immediate cost lies in $[0, 1].$  We also make an assumption that there are $N$ constraints which for each $i \in \{ 1, \dots, N \}, \bar{C}_i \in [0, \bar{C}_{\max}].$
\end{assumption}

Next, to choose an action from $A$ at time-step $h,$ we define a policy $\pi$ as a mapping from state-action space $S \times A$ to set of probability vectors defined over action space, i.e. $\pi : S \times A \rightarrow [0, 1]^{|A|}.$ So $\pi(s, \cdot, h)$ is a probability vector over $A$ at time-step $h.$ Also, $a \sim \pi(s, \cdot, h)$ means that action $a$ is chosen according to policy $\pi$ while being at state $s$ at time-step $h.$


When policy $\pi$ is fixed, the underlying Markov Decision Process turns into a Markov chain.  The transition kernel of this Markov chain is $P_{\pi},$ which can be viewed as an operator. The operator $P_{\pi} f (s) = \mathbb{E} [f(s_{h+1})| s_h = s] = \sum_{s' \in S} P_{\pi}(s'|s) f(s')$ takes any function $f : S \to \mathbb{R}$ and returns the expected value of $f$ in the next time step. For convenience, we define the multi-step version $P_{\pi}^h f (s) = P_{\pi} P_{\pi} \dots P_{\pi} f,$ which is repeated $h$ times. Further, we define $P_{\pi}^{-1}$ and $P_{\pi}^0$ as the identity operator.

We consider cumulative finite horizon criteria for both the objective function and the constraint functions with identical horizon $H.$  We define the value function of state $s$ at time-step $t$ under policy $\pi$ as
\begin{equation}
\label{eq:value}
V^{\pi}_t(s) = \mathbb{E}[\sum_{h=t}^{H-1}  r(s_h, a_h) ; a_h \sim \pi(s_h, \cdot, h), s_t = s],
\end{equation}
where action $a_h$ is chosen according to policy $\pi$ and expectation $\mathbb{E}[.]$ is taken w.r.t transition kernel $P.$ Then, the local variance of the value function at time step $h$ under policy $\pi$ is
\begin{align}
\label{eq:v-local-variance}
    \sigma_h^{\pi^2}(s) = \mathbb{E}[(V^{\pi}_{h+1}(s_{h+1}) - P_{\pi}V^{\pi}_{h+1}(s) )^2].
\end{align}

Similar to the definition of the value function \eqref{eq:value}, the $i^{th}$ constraint function at time $t$ under policy $\pi$ is formulated as
\begin{equation}
\label{eq:cost_i}
C_{i, t}^{\pi}(s) = \mathbb{E}[\sum_{h=t}^{H- 1} c(i, s_h, a_h) ; a_t \sim \pi(s_h, \cdot, h), s_t = s].
\end{equation}
Again, the local variance of $i^{th}$ constraint function at time-step $h$ under policy $\pi,$ i.e.  $\sigma_{i, h}^{\pi^2}$ is defined similar to local variance of value function \eqref{eq:v-local-variance}.

Finally, the general finite-horizon CMDP problem is 
\begin{align}
\label{eq:opt}
\max_{\pi} V^{\pi}_0(s_0) ~~ \text{s.t.} ~~  C_{i, 0}^{\pi}(s_0) \leq \bar{C}_i, \quad \forall i \in \{ 1, \dots, N \}.
\end{align}

\begin{assumption}
\label{a:cmdp-existence}
We assume that there exists some policy $\pi$ that satisfies the constraints in  \eqref{eq:opt}. Hence, this CMDP problem is feasible with optimal policy $\pi^*$ and optimal solution $V^*_0(s_0) = V^{\pi^*}_0(s_0).$
\end{assumption}
Note that we only consider learning feasible CMDPs, since otherwise no algorithm would be able to discover an optimal policy satisfying constraints.


\paragraph{Constrained-RL Problem:} The Constrained RL problem formulation is identical to the CMDP optimization problem of \eqref{eq:opt}, but without being aware of values of transition kernel $P.$\footnote{We only assume that transition kernel is unknown and the extension to unknown reward and cost matrices is straightforward, and does not require additional methodology.} 
Our goal is to provide model-based algorithms and determine the sample complexity results in a PAC sense, which is defined as follows:
\begin{definition}
\label{def:sample-complexity}
For an algorithm $\mathcal{A},$ sample complexity is the number of samples that $\mathcal{A}$ requires to achieve
\begin{align*}
    \mathbb{P} \Big(&V^{\mathcal{A}}_0(s_0) \geq V^{\pi^*}_0(s_0) - \epsilon ~~ \text{and} \\
    &C^{\mathcal{A}}_{i, 0}(s_0) \leq \bar{C}_i + \epsilon~ \forall i \in \{ 1, \dots, N \}\Big) \geq 1 - \delta
\end{align*}
for a given $\epsilon$ and $\delta.$
\end{definition}

Note that this definition includes both objective maximization and constraint violations, as opposed to a traditional definition that only considers the objective~ \cite{mbie}.




\section{Sample Complexity Result of Generative Model Based Learning}
\label{sec:gmbl}

In this section, we introduce a generative model based CMDP learning algorithm called Optimistic Generative Model Based Learning, or Optimistic-GMBL. According to Optimistic-GMBL, we sample each state-action pair $n$ number of times uniformly across all state-action pairs, count the number of times each transition occurs $n(s', s, a)$ for each next state $s',$ and construct an empirical model of transition kernel denoted by $\widehat{P}(s'|s, a) = \frac{n(s', s, a)}{n}~ \forall (s', s, a).$  Then Optimistic-GMBL creates a class of CMDPs using the empirical model. This class is denoted by $\mathcal{M}_{\delta_P}$ and contains CMDPs with identical reward, cost matrices, $\bar{C},$ initial state $s_0$ and horizon of the true CMDP, but with transition kernels close to true model. This class of CMDPs is defined as
\begin{align}
    &\mathcal{M}_{\delta_P} :=  \{ M': r'(s, a) = r(s, a), \label{eq:transition-class}\\
    &c'(i, s, a) = c(i, s, a), H' = H, s'_0 = s_0 \nonumber  \\
    & |P'(s'|s, a) - \widehat{P}(s'|s, a)| \leq \label{eq:bernstein-hoeffding}\\
    &\min \Bigl( \sqrt{\frac{2 \widehat{P}(s'|s, a) (1 - \widehat{P}(s'|s, a) )}{n} \log{\frac{4}{\delta_P}}} + \frac{2}{3n} \log{\frac{4}{\delta_P}}, \nonumber \\
    &\sqrt{\frac{\log{4/\delta_P}}{2n}}\Bigr) \forall s, a, s', i \}, \nonumber
\end{align}
where $\delta_P$ is defined in Algorithm \ref{algo:optimistic-gmbl}. For any $M' \in \mathcal{M},$ objective function $V^{'\pi}_0(s_0)$ and cost functions $C^{'\pi}_{i, 0}(s_0)$ are computed w.r.t. the corresponding transition kernel $P'$ according to equations \eqref{eq:value} and \eqref{eq:cost_i} respectively.

Finally, Optimistic-GMBL maximizes the objective function among all possible transition kernels, while satisfying constraints (if feasible). More specifically, it solves the optimistic planning problem below
\begin{align}
\label{eq:optimistic-planning}
    \max_{\pi, M' \in \mathcal{M}_{\delta_P}} V^{'\pi}_0(s_0) ~~~ \text{s.t.} ~~~ C_{i, 0}^{'\pi}(s_0) \leq \bar{C}_i  ~~ \forall i.
\end{align}

Optimistic-GMBL uses Extended Linear Programming, or \textbf{ELP}, to solve the problem of \eqref{eq:optimistic-planning}. This method inputs $\mathcal{M}_{\delta_P}$ and outputs $\tilde{\pi}$ for the optimal solution. The description of ELP is provided in Appendix \ref{sec:apen} .  
Algorithm \ref{algo:optimistic-gmbl} describes Optimistic-GMBL.

\begin{algorithm}
\caption{Optimistic-GMBL}
\label{algo:optimistic-gmbl}
\begin{algorithmic}[1]
\STATE Input: accuracy $\epsilon$ and failure tolerance $\delta.$
\STATE Set $\delta_P = \frac{\delta}{12(N + 2) |S|^2 |A| H}.$
\STATE Set $n(s', s, a) = 0 ~ \forall (s, a, s').$
\FOR{each $(s, a) \in S \times A$}
\STATE Sample $(s, a), n = \frac{256}{\epsilon^2} |S| H^3 \log{\frac{12(N+2)|S||A|H}{\delta}}$ and update $n(s', s, a).$
\STATE $\widehat{P}(s'|s, a) = \frac{n(s', s, a)}{n} ~ \forall s'.$
\ENDFOR
\STATE Construct $\mathcal{M}_{\delta_P}$ according to \eqref{eq:transition-class}.
\STATE Output $\tilde{\pi} = \text{ELP} (\mathcal{M}_{\delta_P}).$
\end{algorithmic} 
\end{algorithm}

\subsection{PAC Analysis of Optimistic-GMBL}
Here, we present the sample complexity result of Optimistic-GMBL. Time complexity result and analysis will be provided in Appendix \ref{sec:apen}.

\begin{theorem}
\label{thm:gmbl-cmdp-opt-var}
Consider any finite-horizon CMDP $M = \langle S, A, P, r, c, \bar{C}, s_0, H \rangle $ satisfying assumptions \ref{a:cmdp} and \ref{a:cmdp-existence}, and CMDP problem formulation of \eqref{eq:opt}. Then, for any $\epsilon \in (0, \frac{2}{9} \sqrt{\frac{H}{|S|}})$ and $\delta \in (0, 1),$ algorithm \ref{algo:optimistic-gmbl} creates a model CMDP $\tilde{M} = \langle S, A, \tilde{P}, r, c, \bar{C}, s_0, H \rangle$ and outputs policy $\tilde{\pi}$ such that
\begin{align*}
    &\mathbb{P}(V^{\tilde{\pi}}_0(s_0) \geq V^{\pi^*}_0(s_0) - \epsilon ~~ \text{and} ~~\\
    &C^{\tilde{\pi}}_{i, 0}(s_0) \leq \bar{C}_i + \epsilon~\forall i \in \{ 1, 2, \dots, N \} ) \geq 1 - \delta,
\end{align*}
with at least total sampling budget of
\begin{align*}
\frac{256}{\epsilon^2} |S|^2 |A| H^3 \log{\frac{12(N+2)|S||A| H}{\delta}}.
\end{align*}
\end{theorem}

The proof of Theorem \ref{thm:gmbl-cmdp-opt-var} differs from the traditional analysis framework of unconstrained RL~\cite{azar} in the following manner.  First, is the role played by optimism in model construction.  The notion of optimism is not required for learning unconstrained MDPs with generative models, because any estimated model is always feasible \cite{mdp}.  However, there is no such guarantee for any general CMDP problem formulation \cite{altman}. Specifically, simply substituting the true kernel $P$ by the estimated one $\widehat{P}$ is not appropriate, since there is no assurance of feasibility of that problem.   Hence, Optimistic-GMBL converts the CMDP problem under the estimated transition kernel to an optimistic planning problem \eqref{eq:optimistic-planning} and an ELP-based solution.

Second, the core of the analysis of every unconstrained MDP is based on being able to characterize the optimal policy via the Bellman operator.  This technique enables one to obtain a sample complexity that scales with the size of the state space as $O(|S|).$  However, we cannot use this approach to characterize the optimal policy in a CMDP~\cite{altman}.  We require a uniform PAC result over set of all policies and set of value and constraint functions, which in turn leads to $O(|S|^2 \log{|S|})$ sample complexity in the size of state space.

\begin{corollary}
In case of $N = 0,$ the problem would become regular unconstrained MDP. And, the sample complexity result with $N = 0$ would also hold for unconstrained case.
\end{corollary}

Now, we present some of the lemmas that are essential to prove Theorem \ref{thm:gmbl-cmdp-opt-var}. Then we sketch the proof of this theorem. The detailed proofs are provided in Appendix \ref{sec:apen}.

First, we show that true CMDP lies inside the $\mathcal{M}_{\delta_P}$ with high probability, w.h.p. So, the problem \eqref{eq:optimistic-planning} would be feasible w.h.p., since the original CMDP problem is assumed to be feasible according to Assumption \ref{a:cmdp-existence}.

\begin{lemma}
\label{lem:M-in-M_delta}
\begin{align*}
    \mathbb{P}(M \in \mathcal{M}_{\delta_P}) \geq 1 - |S|^2|A| \delta_P.
\end{align*}
\end{lemma}

\textbf{\emph{Proof Sketch:}} Fix a state-action pair $(s,a)$ and next state $s'.$ Then, according to combination of Hoefding's inequality \cite{hoeffding} and empirical Bernstein's inequality \cite{empirical-bernstein}, we get that each $P(s'|s, a)$ is inside the confidence set defined by \eqref{eq:bernstein-hoeffding} with probability at least $1 - \delta_P.$  Applying the union bound yields the result.\hfill $\Box$\vspace{0.05in}

Now, we present the core lemma required for proving Theorem \ref{thm:gmbl-cmdp-opt-var} and its proof sketch. Using this lemma, we bound the mismatch in objective and constraint functions when we have $n$ number of samples from each $(s,a).$   This bound applies uniformly over the set of policies and set of value and constraint functions.  The result also enables us to bound the objective and constraint functions individually.   Then we apply union bound on all objective and constraint functions. This process is the reason why the number of constraints appear logarithmically in the sample complexity result.


\begin{lemma}
\label{lem:v-v'}
Let $\delta_P \in (0, 1).$ 
Then, if $n \geq 2592 |S|^2 H^2 \log{4/\delta_P},$ under any policy $\pi$
\begin{align*}
    \norm{V^{\pi}_0 - \tilde{V}^{\pi}_0}_{\infty} \leq  \sqrt{\frac{32 |S| H^3}{n}}
\end{align*}
w.p. at least $1 - 3|S|^2|A|H\delta_P,$ and for any $i \in \{ 1, \dots, N \},$
\begin{align*}
    \norm{C^{\pi}_{i, 0} - \tilde{C}^{\pi}_{i,0}}_{\infty} \leq \sqrt{\frac{32 |S| H^3}{n}}
\end{align*}
w.p. at least $1 - 3|S|^2|A|H\delta_P.$
\end{lemma}

\textbf{\emph{Proof Sketch:}} We first show that $|\tilde{P}(s'|s, a) - P(s'|s, a)| \leq O(\sqrt{\frac{P(s'|s, a)(1 - P(s'|s,a))}{n})}$ for each $s', s, a.$   Then, we show that at each time-step $h, (P_{\pi} - \tilde{P}_{\pi}) V^{\pi}_h(s) \leq O(\sqrt{\frac{|S|}{n}} \sigma^{\pi}_h(s)).$   Applying this bound to $|\tilde{V}^{\pi}_0(s_0) - V^{\pi}_0(s_0)|$ and from the fact that $\sigma^{\pi}_h(s)$ is close to $\tilde{\sigma}^{\pi}_h(s)$ by $\frac{\sqrt{|S| H^2}}{n^{1/4}},$ we obtain the result. This procedure is also applicable to each constraint function $i.$ \hfill $\Box$\vspace{0.05in}

\textbf{\emph{Proof Sketch of Theorem \ref{thm:gmbl-cmdp-opt-var}}:} From Lemma \ref{lem:M-in-M_delta}, we know that the optimistic planning problem \eqref{eq:optimistic-planning} is feasible w.h.p. Hence, we can obtain an optimistic policy $\tilde{\pi}.$   The rest of this proof consists of two major parts. 

First, we prove $\epsilon-$optimality of objective function w.h.p.  Considering policy $\pi^*$ we obtain $|V^{\pi^*}_0(s_0) - \tilde{V}^{\pi^*}_0(s_0)| \leq O(\sqrt{\frac{|S| H^3}{n}})$ w.h.p. by means of Lemma \ref{lem:v-v'}.  Similarly, $|V^{\tilde{\pi}}_0(s_0) - \tilde{V}^{\tilde{\pi}}_0(s_0)| \leq O(\sqrt{\frac{|S| H^3}{n}})$ w.h.p. Next, we use the fact that $\tilde{V}^{\pi^*}_0(s_0) \leq \tilde{V}^{\tilde{\pi}}_0(s_0)$ and obtain 
\begin{align*}
    V^{\tilde{\pi}}_0(s_0) \geq V^{\pi^*}_0(s_0) - O(\sqrt{\frac{|S| H^3}{n}}).
\end{align*}

Next, we show that each constraint is violated at most by $\epsilon$ w.h.p.  Here, we use the second part of Lemma \ref{lem:v-v'} to bound constraint violation. Thus, for each $i \in \{1, \dots, N \}$ we have $|C^{\tilde{\pi}}_{i, 0}(s_0) - \tilde{C}^{\tilde{\pi}}_{i, 0}(s_0)| \leq O(\sqrt{\frac{|S| H^3}{n}})$ w.h.p.  Also, we know that $\tilde{C}^{\tilde{\pi}}_{i, 0}(s_0) \leq \bar{C}_i,$ since $\tilde{\pi}$ is solution of the ELP. Hence, we obtain
\begin{align*}
    C^{\tilde{\pi}}_{i, 0}(s_0) \leq \bar{C}_i + O(\sqrt{\frac{|S| H^3}{n}})
\end{align*}
w.h.p. Finally, we obtain the end result by applying the union bound, and obtaining $n$ by solving $ \epsilon = O(\sqrt{\frac{|S| H^3}{n}}).$ \hfill $\Box$

\section{Sample Complexity Result of Online Learning}
\label{sec:online}

The Optimistic-GMBL approach requires that every state-action pair in the system be sampled a certain number of times before a policy is computed.  However, many applications may not be able to utilize this approach since it may not be possible to reach those states without the application of some policy, or they might be unsafe and so should not be sampled often.  Hence, we need an approach that can collect samples from the environment by means of an online algorithm.

Online Constrained-RL, or Online-CRL described in Algorithm \ref{algo:online-crl-lp}, is an online method proceeding in episodes with length $H.$ At the beginning of each episode $k,$ Online-CRL constructs an empirical model $\widehat{P}$ according to state-action visitation frequencies, i.e., $\widehat{P}(s'|s, a) = \frac{n(s', s, a)}{n(s, a)},$ where $n(s', s, a)$ and $n(s, a)$ are visitation frequencies. This empirical model $\widehat{P}$ induces a set of finite-horizon CMDPs $\mathcal{M}_k$ which any CMDP $M' \in \mathcal{M}_k$ has identical horizon and reward and cost matrices. However, for any $(s, a) \in S \times A$ and $s' \in S, P'(s'|s, a)$ lies inside a confidence interval induced by $\widehat{P}.$ To construct a confidence interval for any element of $P'(s'|s, a),$ we use identical concentration inequalities to Optimistic GMBL as defined by \eqref{eq:bernstein-hoeffding}. The only difference is the use of $n(s, a)$ instead of $n.$ Thus the class of CMPDs is defined as below at each episode $k:$
\begin{equation}
\label{eq:transition-class-k}
\begin{aligned}
    &\mathcal{M}_{k} :=  \{ M': r'(s, a) = r(s, a),\\
    &c'(i, s, a) = c(i, s, a), H' = H, s'_0 = s_0 \\
    & |P'(s'|s, a) - \widehat{P}(s'|s, a)| \leq \\
    &\min \Bigl( \sqrt{\frac{2 \widehat{P}(s'|s, a) (1 - \widehat{P}(s'|s, a) )}{n(s,a)} \log{\frac{4}{\delta_1}}}\\
    & + \frac{2}{3n(s,a)} \log{\frac{4}{\delta_1}}, \sqrt{\frac{\log{4/\delta_1}}{2n(s,a)}} \Bigr) ~ \forall s,s',a,i \},
\end{aligned}    
\end{equation}
where $\delta_1$ is defined in Algorithm \ref{algo:online-crl-lp}. 

Next, we use ELP to obtain an optimistic policy $\tilde{\pi}_k,$ which is the solution of optimistic CMDP problem below:
\begin{align*}
    \max_{\pi, M' \in \mathcal{M}_k} V^{' \pi}_0(s_0)~~ \text{s.t.}~~ C^{'\pi}_{i, 0}(s_0) \leq \bar{C}_i ~~ \forall ~ i.
\end{align*}

This problem is exactly the same as problem of \eqref{eq:optimistic-planning}, except for substituting $\mathcal{M}_{\delta_P}$ with  $\mathcal{M}_k$. Here, for any $M' \in \mathcal{M}_k,$ $V'^{\pi}_0(s_0)$ and $C'^{\pi}_{i, 0}(s_0)$ are computed according to \eqref{eq:value} and \eqref{eq:cost_i} w.r.t. underlying transition kernel $P',$ respectively.

\begin{algorithm}
\caption{Online-CRL}
\label{algo:online-crl-lp}
\begin{algorithmic}[1]
\STATE Input: accuracy $\epsilon$ and failure tolerance $\delta.$
\STATE Set $k = 1, w_{\min} = \frac{\epsilon}{4H|S|}, U_{\max} = |S|^2 |A| m, \delta_1 = \frac{\delta}{4(N+1)|S|U_{\max}}.$
\STATE Set $m$ according to \eqref{eq:m1} and \eqref{eq:m2}.
\STATE Set $n(s, a) = n(s', s, a) = 0 ~~ \forall s, s' \in S, a \in A.$
\WHILE{there is $(s, a)$ with $n(s, a) < |S|mH$}
\STATE $\widehat{P}(s'|s, a) = \frac{n(s', s, a)}{n(s, a)} ~~ \forall (s, a)$ with $n(s, a) > 0$ and $s' \in S.$
\STATE Construct $\mathcal{M}_k$ according to \eqref{eq:transition-class-k}.
\STATE $\tilde{\pi}_k = \text{ELP}(\mathcal{M}_k).$
\FOR{$t = 1, \dots, H$}
\STATE $a_t \sim \tilde{\pi}_k(s_t), s_{t+1} \sim P(\cdot|s_t, a_t), n(s_t, a_t)++, n(s_{t+1}, s_t, a_t)++.$
\ENDFOR
\STATE $k++$
\ENDWHILE
\end{algorithmic}
\end{algorithm}

This algorithm draws inspiration from the infinite-horizon algorithm UCRL$-\gamma$ \cite{ucrl} and its finite-horizon counterpart UCFH \cite{ucfh} with several differences.  Unlike UCRL-$\gamma$ and UCFH, Algorithm \ref{algo:online-crl-lp} updates the model at the beginning of each episode, which allows for faster model construction.  Also, since we desire a policy that pertains to a CMDP using an linear programming approach~\cite{altman}, we must ensure that all constraints are linear.  Hence, unlike UCFH, Algorithm \ref{algo:online-crl-lp} utilizes a combination of the empirical Bernstein's and Hoeffding's inequalities, which allows us to ensure linearity of constraints (i.e., we can indeed use an extended linear program to solve for the constrained optimistic policy).  However, the constraints of UCFH are non-linear and require the use of extended value iteration coupled with a complex sub-routine, which cannot be utilized in the constrained RL case.  Thus, we are able to obtain strong bounds on sample complexity similar to UCFH, but yet ensure that the solution approach only uses a linear program.


\subsection{PAC Analysis of Online-CRL}

We now present the PAC bound of Algorithm \ref{algo:online-crl-lp}.
\begin{theorem}
\label{thm:online-crl-lp}
Consider CMDP $M = \langle S, A, r, c, \bar{C}, s_0, H \rangle $ satisfying assumptions \ref{a:cmdp} and \ref{a:cmdp-existence}. For any $0 < \epsilon, \delta <1,$  under Online-CRL we have:
\begin{align*}
    &\mathbb{P}(V^{\tilde{\pi}_k}_0(s_0) \geq V^{\pi^*}_0(s_0) - \epsilon ~~ \text{and} ~~\\
    &C^{\tilde{\pi}_k}_{i, 0}(s_0) \leq \bar{C}_i + \epsilon~\forall i \in \{ 1, 2, \dots, N \} ) \geq 1 - \delta,
\end{align*}
for all but at most
\begin{align*}
    \tilde{O}(\frac{|S|^2 |A| H^2}{\epsilon^2} \log{\frac{N+1}{\delta}})
\end{align*}
episodes.
\end{theorem}

To prove Theorem \ref{thm:online-crl-lp}, we follow an approach motivated by \cite{ucrl} and its finite-horizon version \cite{ucfh}. However, there are several differences in our technique.   As mentioned above, one of the differences is with regard to restricting ourselves to only linear concentration inequalities.  We will show that excluding non-linear concentration inequalities pertaining to variance does not increase the sample complexity, and utilizing the fact that the number of successor states is less that $|S|$ leads to matching sample complexity in terms of $|S|$ with the UCFH algorithm.  Furthermore, we are able to show that, unlike existing approaches, we can update the model at each episode, again without increasing the sample complexity.  Thus, we are able to obtain PAC bounds that match the unconstrained case, and only increase by logarithmic factor with the number of constraints.


There are also recent results on characterizing the regret of constrained-RL~\cite{efroni} while using an algorithm reminiscent of Algorithm \ref{algo:online-crl-lp}, and the question arises as to whether one can immediately translate these regret results into sample complexity bounds?  However, regret and sample complexity results do not directly follow from one another~\cite{ubev}, and following the  \cite{efroni} approach gives a PAC result $\tilde{O}(\frac{|S|^2 |A| H^4}{\epsilon^2}),$ which is looser than our result by a factor of $H^2.$  Thus, this alternative option does not provide the strong bounds that we are able to obtain to match existing PAC results of the unconstrained case.


Now, we introduce the notions of \textit{knownness} and \textit{importance} for state-action pairs and base our proof on these notions. Then we present the key lemmas required to prove Theorem \ref{thm:online-crl-lp}. Finally, we sketch the proof of Theorem \ref{thm:online-crl-lp}. The detailed analysis is provided in Appendix \ref{sec:apen}.

Let the \textit{weight} of $(s, a)-$pair in an episode $k$ under policy $\tilde{\pi}_k$ be its expected frequency in that episode
\begin{align*}
    w_k(s, a) &:= \sum_{h = 0}^{H-1} \mathbb{P}(s_h = s, a \sim \tilde{\pi}_k(s_h, \cdot, h))\\
    &= \sum_{h = 0}^{H-1} P_{\tilde{\pi}_k}^{h-1} \mathbb{I} \{ s = \cdot, a \sim \tilde{\pi}_k(s, \cdot, h) \} (s_0).
\end{align*}

Then, the \textit{importance} $\iota_k$ of $(s, a)$ at episode $k$ is defined as its relative weight compared to $w_{\min} := \frac{\epsilon}{4H|S|}$ on a log-scale
\begin{align*}
    &\iota_k (s, a) := \min \{ z_j : z_j \geq \frac{w_k(s, a)}{w_{\min}} \}\\
    &\text{where} ~~ z_1 = 0 ~ \text{and} ~ z_j = 2^{j - 2} ~~ \forall j = 2, 3, \dots.
\end{align*}

Note that $\iota_k (s, a) \in \{ 0, 1, 2, 4, 8, 16, \dots \}$ is an integer indicating the influence of the state-action pair on the value function of $\tilde{\pi}_k.$ Similarly, we define \textit{knownness} as
\begin{align*}
    \kappa_k (s, a) := \max \{ z_i: z_i \leq \frac{n_k(s, a)}{m w_k(s, a)} \} \in \{ 0, 1, 2, 4, \dots \},
\end{align*}
which indicates how often $(s, a)$ has been observed relative to its importance. Value of $m$ is defined in Algorithm \ref{algo:online-crl-lp}. Now, we can categorize $(s, a)-$pairs into subsets
\begin{align*}
    &X_{k, \kappa, \iota} := \{ (s, a) \in X_k: \kappa_k (s, a) = \kappa, \iota_k (s, a) = \iota \}\\
    &\text{and} ~~ \bar{X}_k = S \times A \setminus X_k,
\end{align*}
where $X_k = \{ (s, a) : \iota_k(s, a) > 0 \}$ is the active set and $\bar{X}_k$ is the set of $(s, a)-$pairs that are very unlikely under policy $\tilde{\pi}_k.$ We will show that if $|X_{k, \kappa, \iota}| \leq \kappa$ is satisfied, then the model of Online-CRL would achieve near-optimality while violating constraints at most by $\epsilon$ w.h.p. This condition indicates that important state-action pairs under policy $\tilde{\pi}_k$ are visited a sufficiently large number of times. Hence, the model of Online-CRL will be accurate enough to obtain PAC bounds.

Now, first we show that true model belongs to $\mathcal{M}_k$ for every episode $k$ w.h.p.

\begin{lemma}
\label{lem:admis}
$M \in \mathcal{M}_k$ for all episodes $k$ with probability at least $1 - \frac{\delta}{2(N+1)}.$
\end{lemma}

\textbf{\emph{Proof Sketch:}} Fix a $(s,a),$ next state $s'$ and an episode $k.$ Then, $P(s'|s, a)$ lies inside the confidence set constructed by the combined Bernstein's and Hoeffding's inequalities. Taking the union bound over maximum number of model updates, $U_{\max},$ and next states would yield the result. \hfill $\Box$

Next, we bound the number of episodes that the condition $|X_{k, \kappa, \iota}| \leq \kappa$ is violated w.h.p.

\begin{lemma}
\label{lem:no-of-bad-episodes}
Suppose $E$ is the number of episodes $k$ for which there are $\kappa$ and $\iota$ with $|X_{k, \kappa, \iota}| > \kappa,$ i.e. $E = \sum_{k = 1}^{\infty} \mathbb{I}\{ \exists (\kappa, \iota) : |X_{k, \kappa, \iota}| > \kappa \}$ and let
\begin{align}
    \label{eq:m1}
    m \geq \frac{6H^2}{\epsilon} \log{\frac{2(N+1)E_{\max}}{\delta}},
\end{align}
where $E_{\max} = \log_2 \frac{H}{w_{\min}} \log_2 |S|.$ Then, $\mathbb{P}(E \leq 6 |S| |A| m E_{\max}) \geq 1 - \frac{\delta}{2(N+1)}.$
\end{lemma}

\textbf{\emph{Proof sketch:}} The proof of this lemma is divided into two stages. First, we provide a bound on the total number of times a fixed $(s,a)$ could be observed in a particular $X_{k, \kappa, \iota}$ in all episodes. Then, we present a high probability bound on the number of episodes that $|X_{k, \kappa, \iota}| > \kappa$ for a fixed $(\kappa, \iota).$ Finally, we obtain the result by means of martingale concentration and union bound. \hfill $\Box$ \vspace{0.05in}

Finally, the next lemma provides a bound on the mismatch between objective and constraint functions of the optimistic model and true model.  The role of this lemma is similar to Lemma~\ref{lem:v-v'} for Optimistic-GMBL. It provides a PAC result, which is uniform over value and constraint functions. Hence, it is possible to have individual PAC results for any objective and constraint functions. As discussed in the context of Optimistic-GMBL, this process is responsible for a $\log{N}$ increase in the sample complexity result.

\begin{lemma}
\label{lem:bounded-mismatch}
Assume $M \in \mathcal{M}_k.$ If $|X_{k, \kappa, \iota}| \leq \kappa$ for all $(\kappa, \iota)$ and $0 < \epsilon \leq 1$ and
\begin{align}
    \label{eq:m2}
    m = 1280 \frac{|S| H^2}{\epsilon^2} (\log_2 \log_2 H)^2 \log_2^2 \Bigl( \frac{8|S|^2 H^2}{\epsilon} \Bigr) \log{\frac{4}{\delta_1}},
\end{align}
then $|\tilde{V}^{\tilde{\pi}_k}_0(s_0) - V^{\tilde{\pi}_k}_0(s_0)| \leq \epsilon$ and for any $i, |\tilde{C}^{\tilde{\pi}_k}_{i, 0}(s_0) - C^{\tilde{\pi}_k}_{i, 0}(s_0)| \leq \epsilon.$
\end{lemma}

\textbf{\emph{Proof Sketch:}}  We first use algebraic operations to obtain $|\tilde{P}(s'|s, a) - P(s'|s, a)| \leq O(\sqrt{\frac{P(s'|s, a)(1 - P(s'|s,a))}{n})}$ for each $s', s, a.$  Then we show that at each time-step $h, (P_{\pi} - \tilde{P}_{\pi}) V^{\pi}_h(s) \leq O(\sqrt{\frac{|S|}{n}} \sigma^{\pi}_h(s)).$  Then we divide the state-action based on knownness, i.e., whether they belong to $X_k$ or not.  By applying all bounds and using the fact that $\sigma^{\pi}_h(s)$ is close to $\tilde{\sigma}^{\pi}_h(s)$ by $\frac{\sqrt{|S| H^2}}{n^{1/4}},$  we obtain a bound on $|\tilde{V}^{\pi}_0(s_0) - V^{\pi}_0(s_0)|.$   Eventually, we use the definition of weights to get the final result. This procedure is also applicable to each constraint function $i.$ \hfill $\Box$\vspace{0.05in}

\textbf{\emph{Proof Sketch of Theorem \ref{thm:online-crl-lp}}:} First, we apply Lemma \ref{lem:admis} and show that $M \in \mathcal{M}_k$ for every $k$ w.p. at least $1 - \frac{\delta}{2(N+1)}.$  Therefore, the optimistic planning problem would be feasible and an  optimistic policy $\tilde{\pi}_k$ exists w.h.p.  Furthermore, we bound the number of episodes where $|X_{k, \kappa,\iota}| > \kappa$ w.h.p. by means of Lemma \ref{lem:no-of-bad-episodes}. Thus, for other episodes where $|X_{k, \kappa, \iota}| \leq \kappa,$ we show that objective function is $\epsilon-$optimal and all constraint functions are violated by $\epsilon$ by applying Lemma \ref{lem:bounded-mismatch}. Eventually, taking union bound yields the result. \hfill $\Box$

\section{Experimental Results}
\label{sec:sim}

We conduct experiments on CMDPs akin to a grid world MDP, wherein each square indicates the location of the agent. The goal of the is to start at the fixed start state and reach the final state in $H$ steps.  The agent obtains a reward of $1$ when reaching the goal.  Transitions are stochastic, and given any action, there is probability of self and other transitions, as well as transitioning to other state as intended by the action.   We consider two classes of CMDPs under this setting, namely, (i) state occupancy constraints, and (ii) action frequency constraints, which represent the types of constraints that might appear in real systems.



For the first scenario class, we augment the unconstrained MDP by an action budget constraint. We restrict the number of moves to the right, while ensuring that a feasible path to the goal exists. Here, we consider a $3 \times 3$ and $5\times 5$ grid as examples, with $9$ state states and $25$ states respectively, and with $4$ actions. The $3 \times 3$ and $5 \times 5$ examples are labeled as scenario $1$a and scenario $1$b.

In the second scenario class, we consider a $3 \times 3$ grid world with a particular state is ``bad'' for the CMDP, so the agent must avoid entering it frequently or at all.  The bad state has higher probability of transitioning out of itself compared to the rest of the states. But, if the agent enters this state, a cost is levied. Thus, the constraint is to limit the probability of entering the bad state, and to set the constraint threshold to $0.$ This means that the optimal policy for CMDP is to avoid the bad state altogether.  This process is equivalent to incurring an immediate cost of $1$ when the agent finds itself in the bad state.




We simulate Optimistic-GMBL and Online-CRL for these scenarios. Here, we consider two performance metrics. One, difference in value function calculated by
\begin{align*}
    V^{\pi^*}_0(s_0) - V^{\pi'}_{0}(s_0).
\end{align*}
where $\pi'$ is whether Optimistic-GMBL or Online-CRL. The second performance metric is constraint violation which is calculated by 
\begin{align*}
    \max(C^{\pi'}_0(s_0) - \Bar{C}, 0).
\end{align*}
since we have one constraint in each scenario. Further, we average each data point on every figure over $25$ runs. 

\begin{figure}
    \centering
    \begin{minipage}{0.45 \linewidth}
        \includegraphics[width=\linewidth]{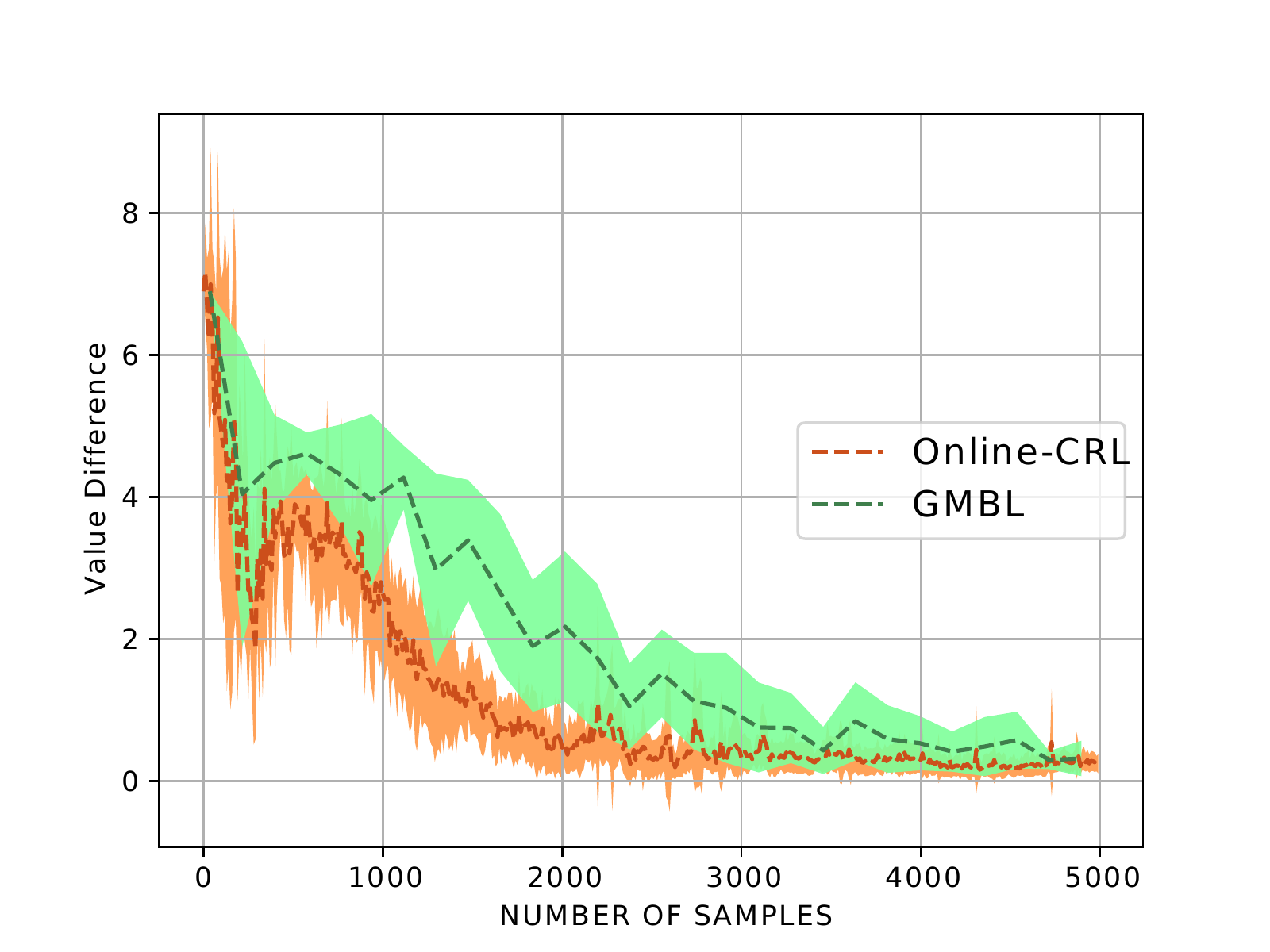}
        \caption{Value Difference for Scenario 1a}
        \label{fig:value1a}
    \end{minipage}
    \begin{minipage}{0.45 \linewidth}
        \includegraphics[width=\linewidth]{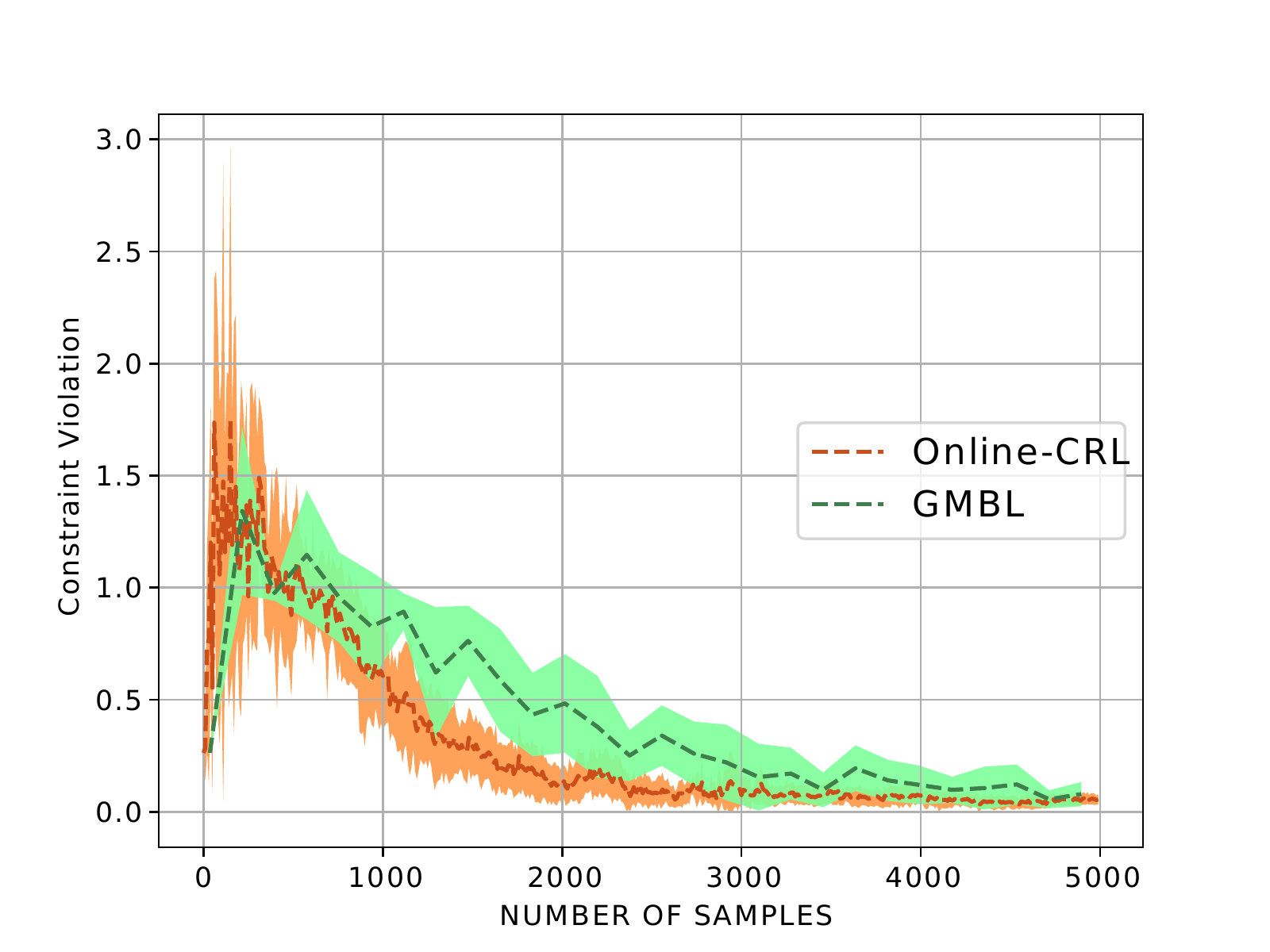}
        \caption{Constraint Violation for Scenario 1a}
        \label{fig:constraint1a}
    \end{minipage}
\end{figure}
\begin{figure}
    \centering
    \begin{minipage}{0.45 \linewidth}
        \includegraphics[width=\linewidth]{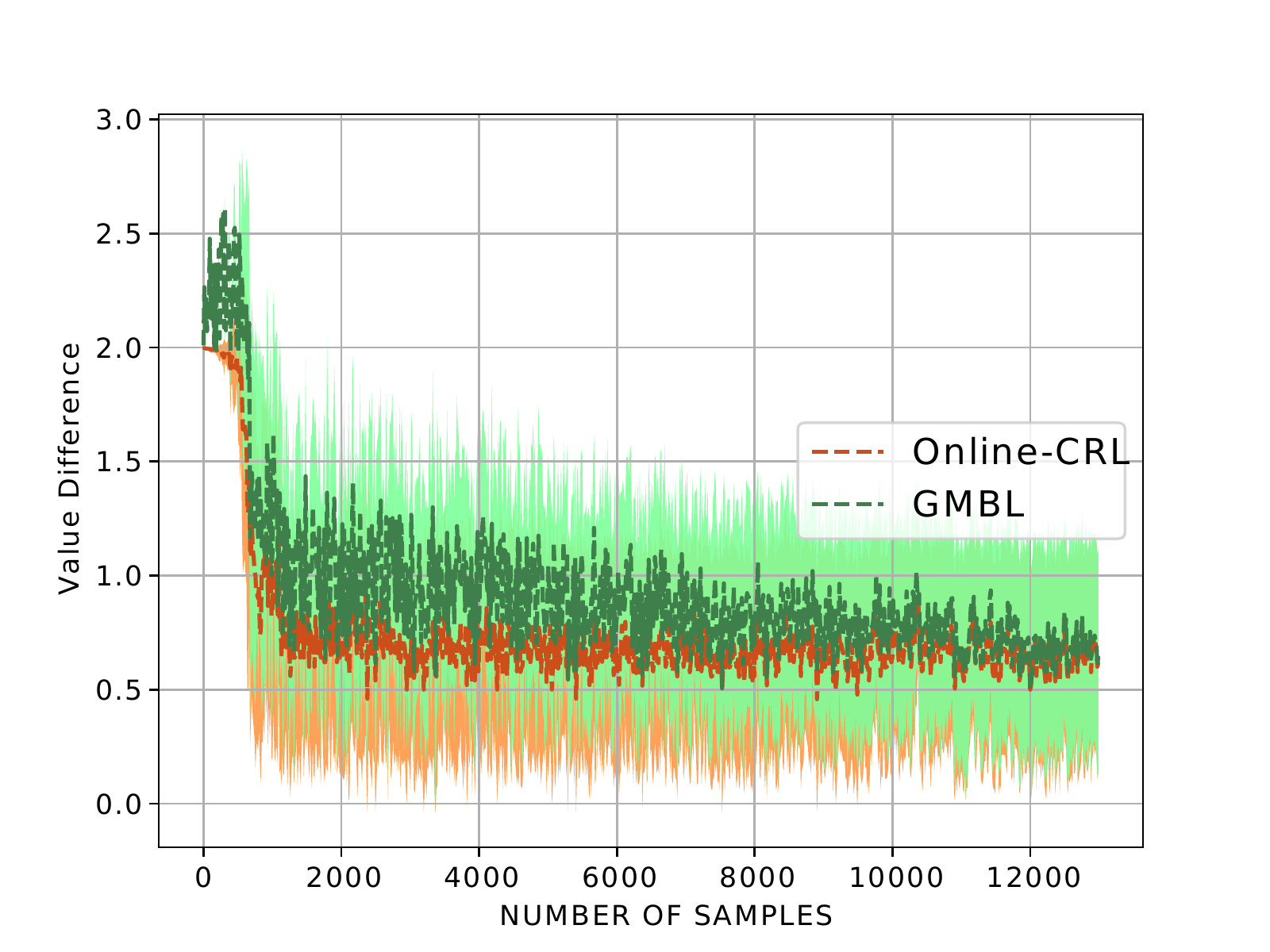}
        \caption{Value Difference for Scenario 1b}
        \label{fig:value1b}
    \end{minipage}
    \begin{minipage}{0.45 \linewidth}
        \includegraphics[width=\linewidth]{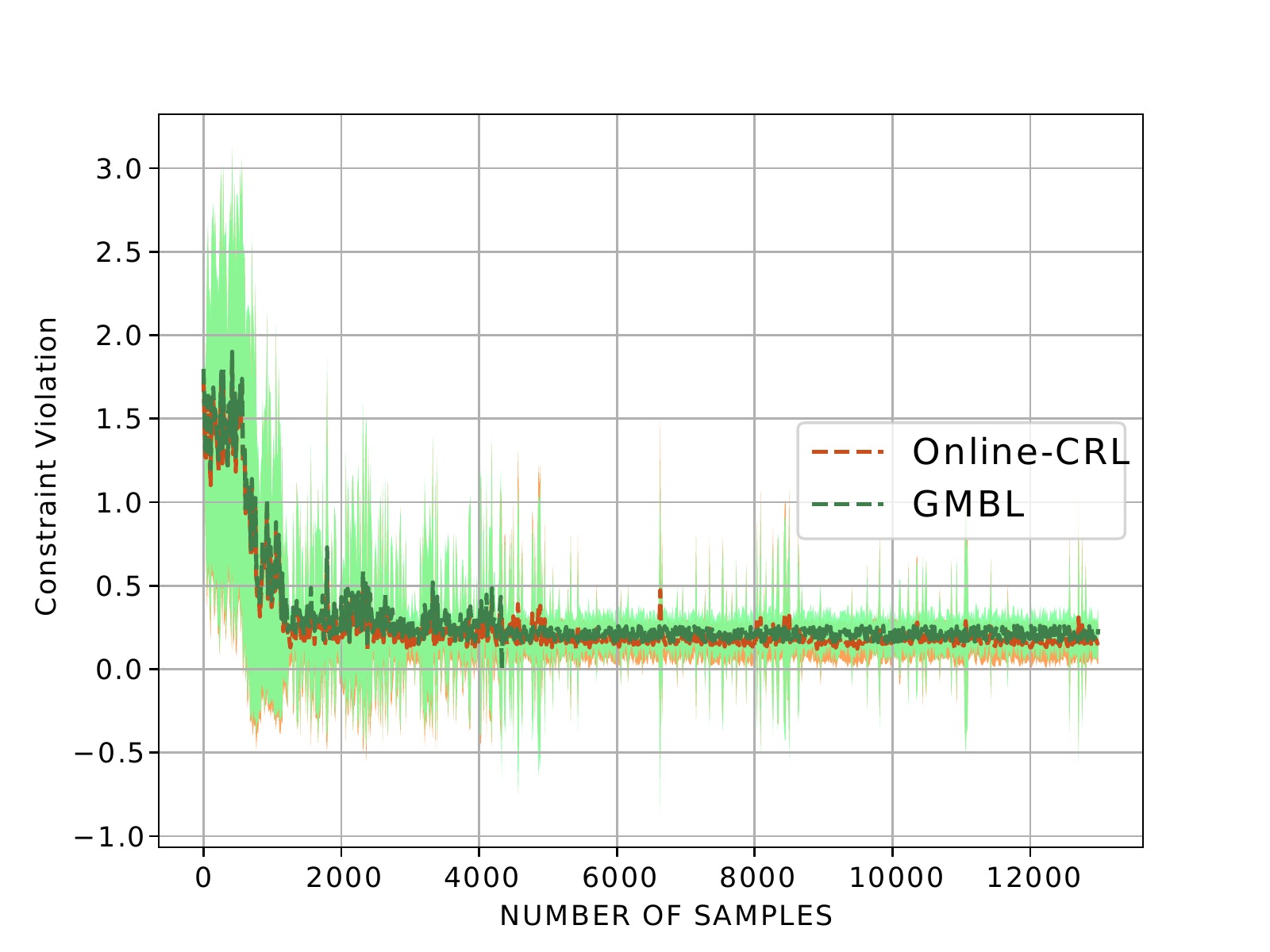}
        \caption{Constraint Violation for Scenario 1b}
        \label{fig:constraint1b}
    \end{minipage}
\end{figure}
\begin{figure}[ht]
    \centering
    \begin{minipage}{0.45 \linewidth}
        \includegraphics[width=\linewidth]{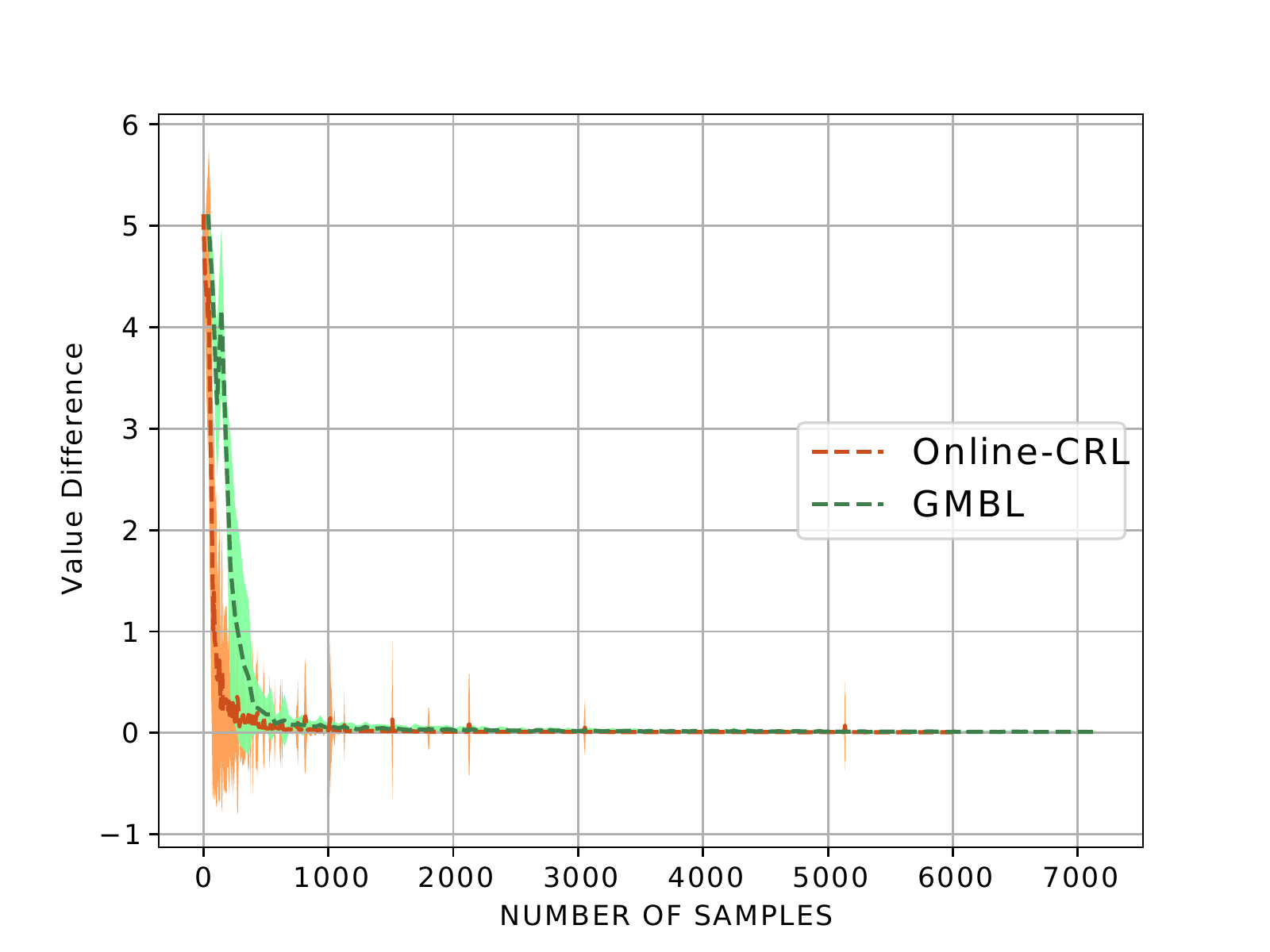}
        \caption{Value Difference for Scenario 2}
        \label{fig:value2}
    \end{minipage}
    \begin{minipage}{0.45 \linewidth}
        \includegraphics[width=\linewidth]{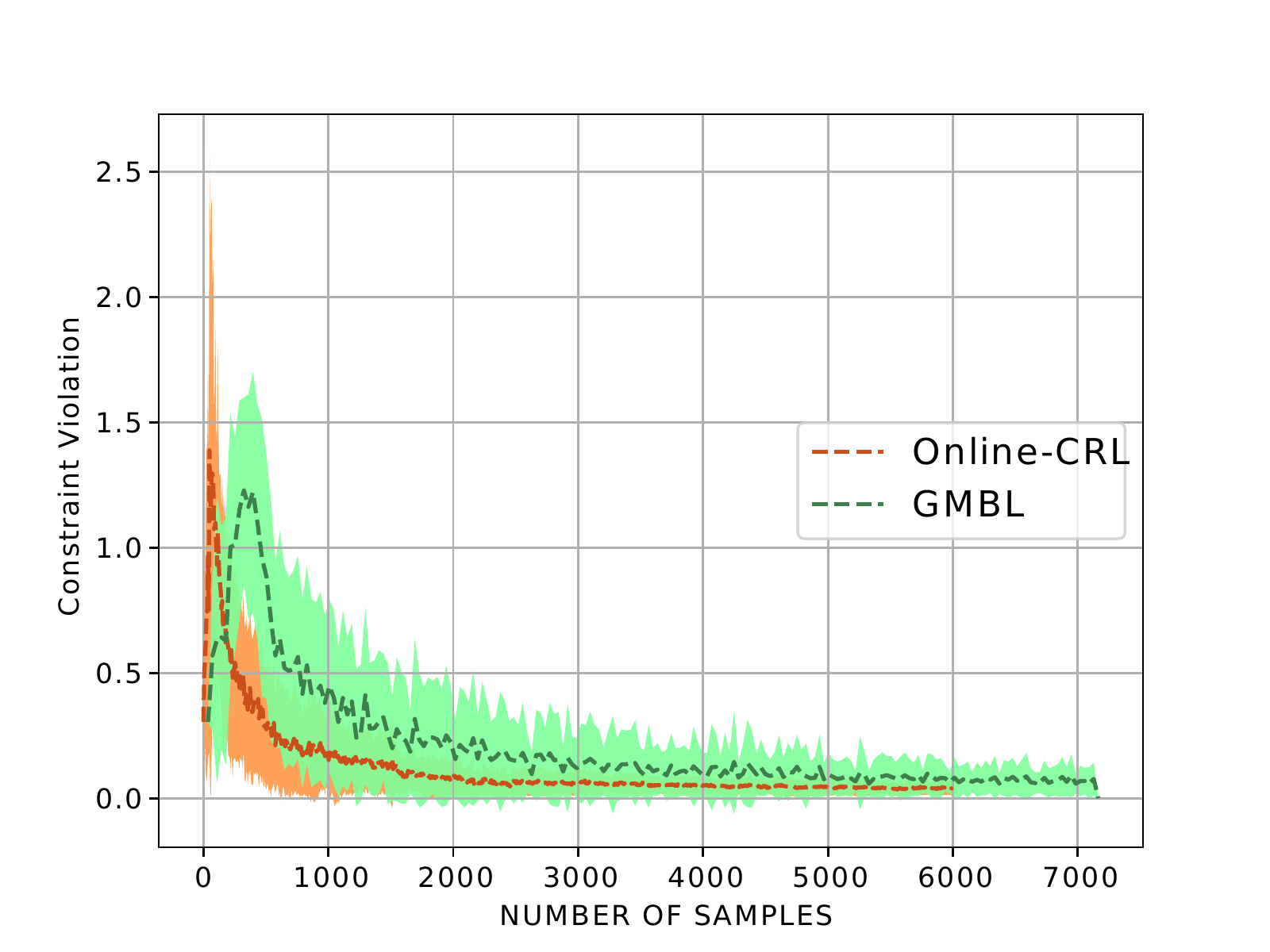}
        \caption{Constraint Violation for Scenario 2}
        \label{fig:constraint2}
    \end{minipage}
\end{figure}

As seen in the Figures \ref{fig:value1a}, \ref{fig:value1b} and \ref{fig:value2}, both Optimistic-GMBL and Online-CRL reach the optimal values in both scenarios. We observe that the Online-CRL algorithm, despite having fewer number of samples, does consistently better than the Optimistic-GMBL algorithm in both the scenarios.  Similar behavior appears in figures \ref{fig:constraint1a}, \ref{fig:constraint1b}, and \ref{fig:constraint2}, which illustrates constraint violation. Intuitively, Online-CRL outperforms Optimistic-GMBL empirically because it samples the important state-action pairs often, and hence resolves uncertainty quickly.

\section{Conclusion}
\label{sec:conc}

This paper introduced the notion of sample complexity in objective maximization and constraint satisfaction for understanding the performance of RL algorithms for safety-constrained applications.  We developed two types of algorithms---Optimistic-GMBL and online-CRL.  
The main finding of a logarithmic factor increase in sample complexity over the unconstrained regime suggests value of the approach to real systems.

\section{Broader Impact}

Reinforcement learning has shown great success in domains that are action constrained, such as robotics, but less so on systems that are safety constrained in terms of the occupancy measure generated by the policy employed.  These include a variety of cyber-physical systems (CPS) such as the power grid, and other utilities, where guarantees on the operating region of the system must be met---ideally deterministically, but within some bounds with high probability in practice.  

It is in the space of control of such CPS that our work is applicable, and could potentially have an impact on a wide variety of supervisory control and data acquisition (SCADA) systems.  Many of them already employ empirically determined policies validated through large scale simulations, and it is not hard to visualize them as being driven by RL-based policies.  Sample complexity bounds reveal how much information is needed to obtain what level of guarantee of safe operability, and hence are a way of determining if a policy has been well enough trained to be actually used.

However, a note of caution with this approach is that the policy generated is only as good as the training environment, and many examples exist wherein the policy generated is optimal according to its training, but violate basic truths known to human operators and could fail quite badly.  Indeed, our approach does not provide sample-path constraints, and the system could well move into deleterious states for a small fraction of the time, which might be completely unacceptable and trigger hard fail safes, such as breakers in a power system.  Understanding the right application environments with excellent domain knowledge is hence needed before any practical success can be claimed.

\bibliography{ref}
\bibliographystyle{unsrt}

\newpage

\section{Technical Appendix}
\label{sec:apen}

\subsection{Extended-Linear Programming}
ELP is a Linear Programming, LP, formulation indeed. So, we first present generic LP which is used to solve CMDP problem of \eqref{eq:opt} \cite{altman}, then build the idea of ELP based on that. To solve CMDP problem \eqref{eq:opt} via LP approach, we convert this problem to a linear programming problem formulated using new variables occupation measures. Now, consider $\mu$ as the finite-horizon state-action occupation measure under policy $\pi$ defined as
\begin{align}
\label{eq:occupancy-measure}
    \mu(s, a, \pi, h) := \mathbb{P}(s_h =s, a_h = a|s_{h = 0} = s_0),
\end{align}
where the probability is calculated w.r.t. underlying transition kernel under policy $\pi, P_{\pi}.$ It is shown that objective function and constraint functions could be restated as functions of occupation measures. Then, the problem would become to find the optimal occupation measures.

Now, if we let $\mu$ be any generic occupation measure defined as \eqref{eq:occupancy-measure}, then the equivalent LP to CMDP problem \eqref{eq:opt} is
\begin{equation}
\label{eq:lp}
     \begin{split}
         \max_{\mu} & \sum_{s, a, h} \mu(s, a, h) r(s, a)\\
         & \text{s.t.}\\
         & \sum_{s, a, h} \mu(s, a, h) c(i, s, a) \leq \bar{C}_i ~~~ \forall i,\\
         & \sum_{a} \mu(s, a, h) = \sum_{s', a'} P(s| s', a') \mu(s', a', h-1) ~~ \forall h \in \{ 1, \dots, H-1 \}, \\
         & \sum_{a} \mu(s_0, a, 0) = 1, ~~ \sum_{a} \mu(s, a, 0) = 0 ~~ \forall s \in S \backslash \{s_0\},\\
         & \mu(s, a, h) \geq 0 ~~ \forall s , a, h \\ 
     \end{split}
 \end{equation}

It is proved that the LP \eqref{eq:lp} is equivalent to CMDP problem of \eqref{eq:opt}, and the optimal policy computed by this LP is also the solution to CMDP problem in \cite{altman}. Eventually, the optimal policy $\pi^*$ is calculated as follows
\begin{align*}
    \pi^*(s, a, h) = \frac{\mu(s, a, h)}{\sum_{b} \mu(s, b, h)}.
\end{align*}

Now, given the estimated model $\widehat{P},$ we get the ELP formulation if we define new occupancy measure $q(s, a, s', h) = P(s'|s, a) \mu(s, a, h).$ Eventually, the ELP formulation is
\begin{align*}
    \max_{q} & \sum_{s, a, s', h} q(s, a, s', h) r(s, a)\\
    & \text{s.t.}\\
    & \sum_{s, a, s', h} q(s, a, s', h) c(i, s, a) \leq \bar{C}_i ~~~ \forall i \in \{1, \dots, N \},\\
    & \sum_{a, s'} q(s, a, s', h) = \sum_{s', a'}  q(s', a', s, h-1) ~~ \forall h \in \{ 1, \dots, H-1 \}, \\
    & \sum_{a, s'} q(s_0, a, s', 0) = 1, ~~ \sum_{a, s'} q(s, a, s', 0) = 0~~\forall s \in S \backslash \{s_0\},\\
    &q(s, a, s', h) \geq 0 ~~ \forall s, s' \in S, a \in A, h \in \{ 0, 1, \dots, H-1 \},\\
    & q(s, a, s', h) - (\widehat{P}(s'|s, a) + \beta(s, a, s')) \sum_{y} q(s, a, y, h) \leq 0~~ \forall s, a, s', h, \\
    & -q(s, a, s', h) + (\widehat{P}(s'|s, a) - \beta(s, a, s')) \sum_{y} q(s, a, y, h) \leq 0 ~~ \forall s, a, s', h,
 \end{align*}
where $\beta(s, a, s')$ is the radius of the confidence interval around $\widehat{P}(s'|s, a)$ which depends on the algorithm. The last two conditions in the above formulation include the confidence interval around $\widehat{P}$ and distinguish ELP from generic LP formulation. At the end, ELP outputs the optimistic policy, $\tilde{\pi}$ for Optimistic-GMBL and $\tilde{\pi}_k$ for Online-CRL, using the solution of above LP. Also, we can calculate an optimistic transition kernel denoted by $\tilde{P}$ by means of optimal $q(s,a,s',h).$ In brief, the optimistic transition kernel and optimistic policy are computed as follows 
\begin{align*}
    \tilde{P}(s'|s, a) = \frac{q(s, a, s', h, s_0)}{\sum_{b} q(s, a, b, h, s_0)}, ~~~ \tilde{\pi}^*(s, a, h) = \frac{\sum_{s'} q(s, a, s', h, s_0)}{\sum_{b, s'} q(s, b, s', h, s_0)}.
\end{align*}

The details of ELP about the time and space complexity is briefed in \cite{efroni}, so we do not present them here.

\subsection{Detailed Proofs for Upper PAC Bounds in Offline Mode}

In this section, we assume that we have $n$ samples from each $(s, a)$ in every lemma presented.

\textbf{\emph{Proof of Lemma \ref{lem:M-in-M_delta}}:} Fix a state, action and next state, i.e. $s, a, s'.$ Then, according to Hoeffding's inequality \cite{hoeffding}
\begin{align*}
    \mathbb{P}(|P(s'|s, a) - \widehat{P}(s'|s, a)| \leq \sqrt{\frac{\log{4/\delta_P}}{2n}}) \geq 1 - \delta_P/2.
\end{align*}
Now, we apply empirical Bernstein's inequality \cite{empirical-bernstein} and get
\begin{align*}
    \mathbb{P}(|P(s'|s, a) - \widehat{P}(s'|s, a)| \leq \sqrt{\frac{2 \widehat{P}(s'|s, a) (1 - \widehat{P}(s'|s, a) )}{n} \log{\frac{4}{\delta_P}}} + \frac{2}{3n} \log{\frac{4}{\delta_P}}) \geq 1 - \delta_P/2.
\end{align*}

By combining these two inequalities and applying union bound, we get
\begin{align*}
    \mathbb{P}(|P(s'|s, a) - \widehat{P}(s'|s, a)| \leq \min \{ \sqrt{\frac{2 \widehat{P}(s'|s, a) (1 - \widehat{P}(s'|s, a) )}{n} \log{\frac{4}{\delta_P}}} + \frac{2}{3n} \log{\frac{4}{\delta_P}}, \sqrt{\frac{\log{4/\delta_P}}{2n}} \} ) \geq 1 - \delta_P.
\end{align*}
Finally, we get the result by applying union bound over all state, action and next states. \hfill  $\Box $


\begin{lemma}
\label{lem:p-phat-ptilde}
Let $\delta_P \in (0, 1).$ Assume $p, \widehat{p}, \tilde{p} \in [0, 1]$ satisfy $\mathbb{P}(p \in \mathcal{P}_{\delta_P}) \geq 1 - \delta_P$ and $\tilde{p} \in \mathcal{P}_{\delta_P}$ where
\begin{align*}
    \mathcal{P}_{\delta_P} := \{ p' \in [0, 1] & : |p' - \widehat{p}| \leq \min \Bigl( \sqrt{\frac{2\widehat{p}(1-\widehat{p})}{n} \log{4/\delta_P}} + \frac{2}{3n} \log{4/\delta_P}, \sqrt{\frac{\log{4/\delta_P}}{2n}} \Bigr) \} .
\end{align*}
Then, 
\begin{align*}
    |p - \tilde{p}| \leq \sqrt{\frac{8 \tilde{p}(1 - \tilde{p})}{n} \log{4/\delta_P}} + 2\sqrt{2} \Bigl( \frac{\log{4/\delta_P}}{n} \Bigr)^{\frac{3}{4}} + 3\sqrt{2} \frac{\log{4/\delta_P}}{n}
\end{align*}
w.p. at least $1 - \delta_P.$
\end{lemma}

\begin{proof}
\begin{align*}
    |p - \tilde{p}| & \leq |p - \widehat{p}| + |\widehat{p} - \tilde{p}| \leq 2 \sqrt{\frac{2\widehat{p}(1 - \widehat{p})}{n} \log{4/\delta_P}} + \frac{4}{3n} \log{4/\delta_P}\\
    &\leq 2 \sqrt{\frac{2 \log{4/\delta_P}}{n} (\tilde{p} + \sqrt{\frac{\log{4/\delta_P}}{2n}})(1 - \tilde{p} + \sqrt{\frac{\log{4/\delta_P}}{2n}})} + \frac{4}{3n} \log{4/\delta_P}\\
    & = 2 \sqrt{\frac{2 \log{4/\delta_P}}{n} \Bigl( \tilde{p}(1 - \tilde{p}) + \sqrt{\frac{\log{4/\delta_P}}{2n}} + \frac{\log{4/\delta_P}}{2n} \Bigr)} + \frac{4}{3n} \log{4/\delta_P}\\
    & \leq \sqrt{\frac{8 \tilde{p}(1 - \tilde{p})}{n} \log{4/\delta_P}} + 2\sqrt{2} \Bigl( \frac{\log{4/\delta_P}}{n} \Bigr)^{\frac{3}{4}} + 3\sqrt{2} \frac{\log{4/\delta_P}}{n}.
\end{align*}
The first term in the first line is true w.p. at least $1 - \delta_P,$ hence the proof is complete.
\end{proof}

\begin{lemma}
\label{lem:v-v'_pi}
Suppose there are two CMDPs $M = \langle S, A, P, r, c, \bar{C}, s_0, H \rangle$ and $M' = \langle S, A, P', r, c, \bar{C}, s_0, H \rangle$ satisfying assumption \ref{a:cmdp}. Then, under any policy $\pi$
\begin{align*}
    V^{\pi}_0 - V'^{\pi}_0  = \sum_{h = 0}^{H - 2} P_{\pi}^{'h-1} (P_{\pi} - P'_{\pi}) V^{\pi}_{h+1} ~~\text{and} ~~ V^{\pi}_0 - V'^{\pi}_0  = \sum_{h = 0}^{H - 2} P_{\pi}^{h-1} (P_{\pi} - P'_{\pi}) V'^{\pi}_{h+1},
\end{align*}
and for any $i \in \{ 1, \dots, N \},$
\begin{align*}
    C^{\pi}_{i, 0} - C'^{\pi}_{i, 0}  = \sum_{h = 0}^{H - 2} P_{\pi}^{'h-1} (P_{\pi} - P'_{\pi}) C^{\pi}_{i, h+1} ~~ \text{and}~~ C^{\pi}_{i, 0} - C'^{\pi}_{i, 0}  = \sum_{h = 0}^{H - 2} P_{\pi}^{h-1} (P_{\pi} - P'_{\pi}) C'^{\pi}_{i, h+1}.
\end{align*}
\end{lemma}

\begin{proof}
We only prove the first statement of value function since the proof procedure for cost is identical. For a fixed $h$ and $s$
\begin{align*}
    &V^{\pi}_h(s) - V'^{\pi}_h(s) = r_{\pi}(s) + \sum_{s'} P_{\pi}(s'| s) V^{\pi}_{h+1} (s') - (r_{\pi}(s) + \sum_{s'} P'_{\pi}(s'| s) V'^{\pi}_{h+1} (s'))\\
    & = \sum_{s'} P_{\pi}(s'| s) V^{\pi}_{h+1} (s') - \sum_{s'} P'_{\pi}(s'| s) V^{\pi}_{h+1} (s') + \sum_{s'} P'_{\pi}(s'| s) V^{\pi}_{h+1} (s') - \sum_{s'} P'_{\pi}(s'| s) V'^{\pi}_{h+1} (s')\\
    & = \sum_{s'} (P_{\pi}(s'|s) - P'_{\pi}(s'|s)) V^{\pi}_{h+1}(s') + \sum_{s'} P'_{\pi}(s'|s) (V^{\pi}_{h+1}(s') - V'^{\pi}_{h+1}(s')).
\end{align*}
Because $V^{\pi}_{H-1}(s) = V'^{\pi}_{H-1}(s) = r_{\pi}(s),$ if we expand the second term until $h = H-1,$ we get the result.
\end{proof}

\begin{lemma}
\label{lem:bound_on_pv}
Let $\delta_P \in (0, 1).$ Suppose there are two CMDPs $M = \langle S, A, P, r, c, \bar{C}, s_0, H \rangle $ and $M' = \langle S, A, P', r, c, \bar{C}, s_0, H \rangle$ satisfying assumption \ref{a:cmdp}. Further assume
\begin{align*}
    |P(s'|s, a) - P'(s'|s, a)| \leq c_1 + c_2 \sqrt{P'(s'|s, a) - (1 - P'(s'|s, a))}
\end{align*}
w.p. at least $1 - \delta_P$ for each $s, s' \in S, a \in A.$ Then, under any policy $\pi$
\begin{align*}
    |\sum_{s'} (P_{\pi}(s'|s) - P'_{\pi}(s'|s)) V'^{\pi}_{h+1}(s')| \leq |S| c_1 \norm{V'^{\pi}_{h+1}}_{\infty} + c_2 \sqrt{|S|} \sigma'^{\pi}_h(s)
\end{align*}
for any $(s, a) \in S \times A$ and $h \in [0, H-2]$ w.p. at least $1 - |S| \delta_P,$ and
\begin{align*}
    |\sum_{s'} (P_{\pi}(s'|s) - P'_{\pi}(s'|s)) C'^{\pi}_{i, h+1}(s')| \leq |S| c_1 \norm{C'^{\pi}_{i, h+1}}_{\infty} + c_2 \sqrt{|S|} \sigma'^{\pi}_{i, h}(s)
\end{align*}
for any $(s, a) \in S \times A, i \in \{ 1, \dots, N \}$ and $h \in [0, H-2]$ w.p. at least $1 - |S| \delta_P.$
\end{lemma}

\begin{proof}
We only prove the statement of value function since the proof procedure for cost is identical. Fix state $s$ and define for this fixed state $s$ the constant function $\bar{V}^{\pi}(s') = \sum_{s''} P'_{\pi}(s''|s) V'^{\pi}_{h+1}(s'')$ as the expected value function of the successor states of $s.$ Note that $\bar{V}^{\pi}(s')$ is a constant function and so $\bar{V}^{\pi}(s') = \sum_{s''} P'_{\pi}(s''|s) \bar{V}^{\pi}(s'') = \sum_{s''} P_{\pi}(s''|s) \bar{V}^{\pi}(s'').$
\begin{align}
    & |\sum_{s'} (P_{\pi}(s'|s) - P'_{\pi}(s'|s)) V'^{\pi}_{h+1}(s')| = |\sum_{s'} (P_{\pi}(s'|s) - P'_{\pi}(s'|s)) V'^{\pi}_{h+1}(s') + \bar{V}^{\pi}(s) - \bar{V}^{\pi}(s)| \nonumber \\
    & = |\sum_{s'} (P_{\pi}(s'|s) - P'_{\pi}(s'|s)) (V'^{\pi}_{h+1}(s') - \bar{V}^{\pi}(s'))| \nonumber \\
    & \leq \sum_{s'} |P_{\pi}(s'|s) - P'_{\pi}(s'|s)| |V^{\pi}_{h+1}(s') - \bar{V}^{\pi}(s')| \label{eq:inter1} \\
    & \leq \sum_{s'} (c_1 + c_2 \sqrt{P'_{\pi}(s'|s) - (1 - P'_{\pi}(s'|s))}) |V^{\pi}_{h+1}(s') - \bar{V}^{\pi}(s')| \nonumber \\
    & \leq |S| c_1 \norm{V'^{\pi}_{h+1}}_{\infty} + c_2 \sum_{s'} \sqrt{P'_{\pi}(s'|s) (1 - P'_{\pi}(s'|s)) (V^{\pi}_{h+1}(s') - \bar{V}^{\pi}(s'))^2} \nonumber \\
    & \leq |S| c_1 \norm{V'^{\pi}_{h+1}}_{\infty} + c_2 \sqrt{|S| \sum_{s'} P'_{\pi}(s'|s) (1 - P'_{\pi}(s'|s)) (V^{\pi}_{h+1}(s') - \bar{V}^{\pi}(s'))^2} \label{eq:inter2}\\
    & \leq |S| c_1 \norm{V'^{\pi}_{h+1}}_{\infty} + c_2 \sqrt{|S| \sum_{s'} P'_{\pi}(s'|s) (V^{\pi}_{h+1}(s') - \bar{V}^{\pi}(s'))^2} \nonumber\\
    & = |S| c_1 \norm{V'^{\pi}_{h+1}}_{\infty} + c_2 \sqrt{|S|} \sigma'^{\pi}_{h}. \nonumber
\end{align}
Inequality \eqref{eq:inter1} holds w.p. at least $1 - |S| \delta_P$, since we used the assumption and applied the triangle inequality and union bound. We then applied the assumed bound on $|V'^{\pi}_{h+1}(s') - \bar{V}^{\pi}(s')|$ and bounded it by $\norm{V'^{\pi}_{h+1}}_{\infty}$ as all value functions are non-negative. In inequality \eqref{eq:inter2}, we applied the Cauchy-Schwarz inequality and subsequently used the fact that each term is the sum is non-negative and that $(1 - P'_{\pi}(s'|s)) \leq 1.$ The final equality follows from the definition of $\sigma'^{\pi}_{h}(s).$
\end{proof}

\begin{lemma}
\label{lem:trivial_bound_vpi}
Let $\delta_P \in (0, 1).$ Suppose there are two CMDPs $M = \langle S, A, P, r, c, \bar{C}, s_0, H \rangle $ and $M' = \langle S, A, P', r, c, \bar{C}, s_0, H \rangle$ satisfying assumption \ref{a:cmdp}. Further assume
\begin{align*}
    |P(s'|s, a) - P'(s'|s, a)| \leq \frac{a}{\sqrt{n}}
\end{align*}
for all $s, s' \in S, a \in A$ w.p. at least $1 - \delta_P.$ Then, under any policy $\pi$
\begin{align*}
        \norm{V^{\pi}_{H-1} - V'^{\pi}_{H-1}}_{\infty} \leq \dots \leq \norm{V^{\pi}_0 - V'^{\pi}_0}_{\infty} \leq |S|H^2 a \frac{1}{\sqrt{n}},
\end{align*}
w.p. at least $1 - |S|^2|A|H\delta_P,$ and for any $i \in \{ 1, \dots, N \}$
\begin{align*}
    \norm{C^{\pi}_{i, H-1} - C'^{\pi}_{i, H-1}}_{\infty} \leq \dots \leq \norm{C^{\pi}_{i, 0} - C'^{\pi}_{i, 0}}_{\infty} \leq |S|H^2 a \frac{1}{\sqrt{n}}
\end{align*}
w.p. at least $1 - |S|^2|A|H\delta_P.$
\end{lemma}

\begin{proof}
We prove the statement of value function since the proof procedure for cost is identical. Let $\Delta_h = \max_s |V^{\pi}_h(s) - V'^{\pi}_h(s)|.$ Then
\begin{align*}
    & \Delta_h = |V^{\pi}_h(s) - V'^{\pi}_h(s)| = |r_{\pi}(s) + \sum_{s'} P_{\pi}(s'|s) V^{\pi}_{h+1}(s') - (r_{\pi}(s) + \sum_{s'} P'_{\pi}(s'|s) V'^{\pi}_{h+1}(s'))| \\
    & = |\sum_{s'} P_{\pi}(s'|s) V^{\pi}_{h+1}(s') - \sum_{s'} P'_{\pi}(s'|s) V^{\pi}_{h+1}(s') + \sum_{s'} P'_{\pi}(s'|s) V^{\pi}_{h+1}(s') - \sum_{s'} P'_{\pi}(s'|s) V'^{\pi}_{h+1}(s')|\\
    & \leq \sum_{s'} |(P_{\pi}(s'|s) - P'_{\pi}(s'|s)| H + \Delta_{h+1}\\
    & \leq |S| H a \frac{1}{\sqrt{n}} + \Delta_{h+1}.
\end{align*}
Thus,
\begin{align*}
    \Delta_{h} \leq |S| H a \frac{1}{\sqrt{n}} + \Delta_{h+1}
\end{align*}
w.p. at least $1 - |S|^2|A| \delta_P$ by applying union bound over all current state, action and next state. If we expand this recursively, we get
\begin{align*}
    \Delta_{H-1} = 0 \leq \dots \leq \Delta_0 \leq |S|H^2 a \frac{1}{\sqrt{n}}
\end{align*}
since $\Delta_{H-1} = \max_{s}|r_{\pi}(s) - r_{\pi}(s)| = 0.$ By taking union bound over time-steps, we get the result holds w.p. at least $1 - |S|^2|A|H\delta_P.$ Hence the proof is complete.
\end{proof}

\begin{lemma}
\label{lem:sigma-sigma'}
Let $\delta_P \in (0, 1).$ Suppose there are two CMDPs $M = \langle S, A, P, r, c, \bar{C}, s_0, H \rangle$ and $M' = \langle S, A, P', r, c, \bar{C}, s_0, H \rangle$ satisfying assumption \ref{a:cmdp}. Further assume
\begin{align*}
    |P(s'|s, a) - P'(s'|s, a)| \leq \frac{a}{\sqrt{n}}
\end{align*}
w.p. at least $1 - \delta_P$ for all $s, s' \in S, a \in A.$ Then if $n \geq a|S|H^2,$ at any time-step $h \in [0, H-1]$ and under any policy $\pi$
\begin{align*}
    \norm{\sigma^{\pi}_h - \sigma'^{\pi}_h}_{\infty} \leq  \frac{2\sqrt{2|S|H^2 a}}{n^{1/4}},
\end{align*}
w.p. at least $1 - 2|S|^2|A|H\delta_P,$ and similarly for any $i \in \{ 1, \dots, N \}$
\begin{align*}
    \norm{\sigma^{\pi}_{i, h} - \sigma'^{\pi}_{i, h}}_{\infty} \leq \frac{2\sqrt{2|S|H^2 a}}{n^{1/4}}
\end{align*}
w.p. at least $1 - 2|S|^2|A|H\delta_P.$
\end{lemma}

\begin{proof}
We prove the statement of value function since the proof procedure for cost is identical. Fix a state $s.$ Then,
\begin{align*}
    \sigma_h^{\pi^2}(s) & = \sigma_h^{\pi^2} (s) - \mathbb{E}'[(V^{\pi}_{h+1}(s_{h+1}) - P'_{\pi}V^{\pi}_{h+1}(s) )^2] + \mathbb{E}'[(V^{\pi}_{h+1}(s_{h+1}) - P'_{\pi}V^{\pi}_{h+1}(s) )^2]\\
    & \leq \sum_{s'} (P_{\pi}(s'|s) - P'_{\pi}(s'|s)) V^{\pi^2}_{h+1}(s') - [(\sum_{s'} P_{\pi}(s'|s) V^{\pi}_{h+1}(s'))^2 - (\sum_{s'} P'_{\pi} (s'|s) V^{\pi}_{h+1}(s'))^2]\\
    & + [\sqrt{ \mathbb{E}' [(V^{\pi}_{h+1}(s_{h+1}) - V'^{\pi}_1(s_1) - P'_{\pi}(V^{\pi}_{h+1} - V'^{\pi}_{h+1})(s) )^2]} + \sqrt{\mathbb{E}'[(V'^{\pi}_{h+1}(s_{h+1}) - P'_{\pi}(V'^{\pi}_{h+1})(s))^2}]^2,
\end{align*}
where we applied triangular inequality in the last line. And, please note that $\mathbb{E}'$ means expectation w.r.t. transition kernel $P'_{\pi}.$ It is straightforward to show that $Var_{s' \sim P'_{\pi}(\cdot|s)}(V^{\pi}_h(s') - V'^{\pi}_h(s')) \leq \norm{V^{\pi}_h - V'^{\pi}_h}_{\infty}^2$ implying
\begin{align*}
    \sigma_h^{\pi^2}(s) & \leq \sum_{s'} (P_{\pi}(s'|s) - P'_{\pi}(s'|s)) V^{\pi^2}_{h+1}(s')\\
    & - [\sum_{s'} (P_{\pi}(s'|s) - P'_{\pi}(s'|s)) V^{\pi}_{h+1}(s')] [\sum_{s'} (P_{\pi}(s'|s) + P'_{\pi}(s'|s)) V^{\pi}_{h+1}(s')]\\
    & + (\norm{V^{\pi}_h - V'^{\pi}_h}_{\infty} + \sigma'^{\pi}_h(s))^2
\end{align*}
w.p. at least $1 - |S|\delta_P$ Now, if we use Lemma \ref{lem:trivial_bound_vpi}, we get
\begin{align*}
    \sigma_h^{\pi^2}(s) &\leq [\sigma_h^{'\pi}(s) + \frac{|S|H^2 a}{\sqrt{n}}]^2 + \frac{2|S|aH^2}{\sqrt{n}} \leq [\sigma_h^{'\pi}(s) + \frac{|S|H^2 a}{\sqrt{n}} + \frac{\sqrt{2|S|H^2a}}{n^{1/4}}]^2\\
    & \leq  [\sigma_h^{'\pi}(s) + \frac{2\sqrt{2|S|H^2 a}}{n^{1/4}}]^2,
\end{align*}
w.p. at least $1 - |S|^2|A| H \delta_P.$\footnote{Please note that when the assumption on transition kernel holds, then $\sum_{s'} (P_{\pi}(s'|s) - P'_{\pi}(s'|s)) V^{\pi^2}_{h+1}(s')$ and $\norm{V^{\pi}_h - V'^{\pi}_h}_{\infty}$ are dependent. And, we can consider the one with lower probability.}  In the last line, we used the fact that for any $x, y > 0$ we have $x^2 + y^2 \leq (x + y)^2.$ And, the assumption on $n,$ dominates the term with $\frac{1}{n^{1/4}}$ over $\sqrt{n}.$ Eventually, the result follows by taking square root from both sides and union bound on both directions, i.e. $\sigma'^{\pi}_h(s) \leq \sigma^{\pi}_h(s) + \frac{2\sqrt{2|S|H^2a}}{n^{1/4}}.$ \footnote{Here, we also know that the high probability bound on $|\sigma^{\pi}_h(s) - \sigma'^{\pi}_h(s)|$ is dependent over all $(s, a).$}
\end{proof}

\begin{lemma} \cite{ucfh}
\label{lem:variance-bound}
The variance of the value function defined as $\Sigma_t^{\pi}(s)= \mathbb{E} [(\sum_{h = t}^{H-1} r(s_h) - V^{\pi}_0(s) )^2]$ satisfies a Bellman equation $\Sigma_t^{\pi}(s) = \sigma_{t}^{\pi^2}(s) + \sum_{s' \in S} P_{\pi}(s'|s) V^{\pi}_{t+1} (s')$ which gives $\Sigma_t^{\pi}(s) = \sum_{h = t}^{H} (P^{h-1}_{\pi} \sigma_h^{\pi^2}) (s).$ Since $0 \leq \Sigma_0^{\pi}(s) \leq H^2,$ it follows that $0 \leq \sum_{h = 0}^{H-1} (P^{h-1}_{\pi} \sigma_h^{\pi^2}) (s) \leq H^2$ for all $s \in S.$
\end{lemma}

\begin{corollary}
The result of Lemma \ref{lem:variance-bound} also holds for variance of cost functions.
\label{cor:variance-bound}
\end{corollary}

\textbf{\emph{Proof of Lemma \ref{lem:v-v'}:}} We only prove the statement of value function since the proof procedure for cost is identical. First, we apply Lemma \ref{lem:p-phat-ptilde} and get
\begin{align*}
    |P(s'|s, a) - \tilde{P}(s'|s, a)| \leq \sqrt{\frac{8 \tilde{P(s'|s, a)}(1 - \tilde{P}(s'|s, a))}{n} \log{4/\delta_P}} + 2\sqrt{2} \Bigl( \frac{\log{4/\delta_P}}{n} \Bigr)^{\frac{3}{4}} + 3\sqrt{2} \frac{\log{4/\delta_P}}{n}
\end{align*}
w.p. at least $1 - \delta_P.$ So, let
\begin{align}
\label{eq:c1-c2}
    c_1 = 2\sqrt{2} \Bigl( \frac{\log{4/\delta_P}}{n} \Bigr)^{\frac{3}{4}} + \frac{3\sqrt{2} \log{4/\delta_P}}{n} ~~ \text{and} ~~ c_2 = \sqrt{\frac{8\log{4/\delta_P}}{n}}
\end{align}

Now, let fix state $s:$
\begin{align}
    & |V^{\pi}_0(s) - \tilde{V}^{\pi}_0(s)| = |\sum_{h = 0}^{H-2} \tilde{P}^{h-1}_{\pi} (P_{\pi} - \tilde{P}_{\pi}) V^{\pi}_{h+1} |(s) \label{eq:l1}\\
    & \leq \sum_{h = 0}^{H-2} \tilde{P}^{h-1}_{\pi} |(P_{\pi} - \tilde{P}_{\pi}) V^{\pi}_{h+1}|(s) \leq \sum_{h = 0}^{H-2} \tilde{P}^{h-1}_{\pi} (|S| c_1 \norm{V^{\pi}_{h+1}}_{\infty} + c_2 \sqrt{|S|} \sigma^{\pi}_h)(s) \label{eq:l2}\\
    & \leq |S| H^2 c_1 + c_2 \sqrt{|S|} \sum_{h = 0}^{H-1} (\tilde{P}^{h-1}_{\pi} \sigma^{\pi}_h)(s) \label{eq:l3}\\
    & \leq |S|H^2 c_1 + c_2 \sqrt{|S|} \sum_{h=0}^{H-1} (\tilde{P}^{h-1} (\tilde{\sigma}_h^{\pi} + \frac{2^{1.25} |S|^{0.5} H (\log{4/\delta_P})^{0.25} }{n^{1/4}})(s) \label{eq:l4}\\
    & \leq |S|H^2 c_1 + c_2 \sqrt{|S|H} \sqrt{\sum_{h = 0}^{H-1} (\tilde{P}^{h-1} \tilde{\sigma}^{\pi^2}_h) (s) } + c_2 H \sqrt{|S|} \frac{2^{1.25} |S|^{0.5} H (\log{4/\delta_P})^{0.25}}{n^{1/4}} \label{eq:l5}\\
    & = \frac{3\sqrt{2} |S| H^2 \log{4/\delta_P}}{n} + \frac{2\sqrt{2} |S| H^2 (\log{4/\delta_P})^{\frac{3}{4}}}{n^{\frac{3}{4}}} + \sqrt{\frac{8 |S| H^3 \log{4/\delta_P}}{n}} + \frac{2^{2.75} |S| H^2 (\log{4/\delta_P})^{\frac{3}{4}}}{n^{\frac{3}{4}}} \label{eq:l6}\\
    & \leq \sqrt{32\frac{|S| H^3}{n}}. \label{eq:l7}
\end{align}
In equation \eqref{eq:l1}, we used Lemma \ref{lem:v-v'_pi}. Then, we applied Lemma \ref{lem:bound_on_pv} to obtain inequality \eqref{eq:l2}. Next, we bound $\norm{V^{\pi}_{h+1}}_{\infty}$ by $H$ in inequality \eqref{eq:l3}. To get inequality \eqref{eq:l4}, we use Lemma \ref{lem:sigma-sigma'}, since we can bound $P(\cdot| s, a) - \tilde{P}(\cdot|s, a)$ by $c_2.$ And, we applied Cauchy-Scharwz inequality to get inequality \eqref{eq:l5}. To get inequality \eqref{eq:l6}, we applied Lemma \ref{lem:variance-bound} and substituting $c_1$ and $c_2$ according to equations \eqref{eq:c1-c2}. Finally, inequality \eqref{eq:l7} follows from the fact that $n \geq 2592 |S|^2 H^2 \log{4/\delta_P}.$ Since the result is true for every $s \in S,$ hence the proof is complete. \hfill  $\Box $

\paragraph{Proof of Theorem \ref{thm:gmbl-cmdp-opt-var}:} Let $\delta_P \in (0, 1).$ First, we know that optimistic planning problem \eqref{eq:optimistic-planning} is feasible w.p. at least $1 - |S|^2|A|\delta_P.$ The following events are dependent on this event. Thus, we consider the lowest probability of feasibility and following events.

Now,  we have
\begin{align*}
    V^{\pi^*}_0(s_0) - \sqrt{32\frac{|S| H^3 \log{4/\delta_P}}{n}} \leq \tilde{V}^{\pi^*}_0(s_0) \leq V^{\pi^*}_0(s_0) + \sqrt{\frac{32 |S| H^3 \log{4/\delta_P}}{n}}
\end{align*}
w.p. at least $1 - 3|S|^2|A|H \delta_P$ and
\begin{align*}
    V^{\tilde{\pi}}_0(s_0) - \sqrt{\frac{32|S|H^3 \log{4/\delta_P}}{n}} \leq \tilde{V}^{\tilde{\pi}}_0(s_0) \leq V^{\tilde{\pi}}_0(s_0) + \sqrt{\frac{32 |S|H^3 \log{4/\delta_P}}{n}}
\end{align*}
w.p. at least $1 - 3|S|^2|A|H \delta_P$ according to Lemma \ref{lem:v-v'}. On the other hand, we know that $\tilde{V}^{\pi^*}_0(s_0) \leq \tilde{V}^{\tilde{\pi}}_0(s_0).$ Thus, by combining these results we get
\begin{align*}
    V^{\pi^*}_0(s_0) - \sqrt{\frac{32|S|H^3 \log{4/\delta_P}}{n}} \leq \tilde{V}^{\pi^*}_0(s) \leq \tilde{V}^{\tilde{\pi}}_0(s_0) \leq V^{\tilde{\pi}}_0(s) + \sqrt{\frac{32|S|H^3 \log{4/\delta_P}}{n}}.
\end{align*}
It yields that $V^{\tilde{\pi}}_0(s_0) \geq V^{\pi^*}_0(s_0) - 2\sqrt{\frac{32|S|H^3 \log{4/\delta_P}}{n}}$ w.p. at least $1 - 6|S|^2|A|H\delta_P$ by union bound.

On the other hand, for any $i \in \{ 1, \dots, N \}$ we have
\begin{align*}
    C^{\tilde{\pi}}_{i, 0}(s_0) \leq \tilde{C}^{\tilde{\pi}}_{i, 0}(s_0) + \sqrt{\frac{32|S|H^3 \log{4/\delta_P}}{n}} \leq \bar{C}_i + \sqrt{\frac{32 |S|H^3 \log{4/\delta_P}}{n}}
\end{align*}
w.p. at least $1 - 3|S|^2|A| H \delta_P$ according to Lemma \ref{lem:v-v'}. By taking union bound, we get that all statements for value and cost functions hold w.p. at least $1 - (3N+6) |S|^2 |A| H \delta_P.$ Hence, putting $\epsilon = 2\sqrt{\frac{32|S|H^3 \log{4/\delta_P}}{n}}$ and $\delta = 12(N+2) |S|^2 |A| H \delta_P$ concludes the proof. Please note that $\epsilon < \frac{2}{9}\sqrt{\frac{H}{|S|}}$ would satisfy the assumption in Lemma \ref{lem:v-v'}. \hfill  $\Box$

\subsection{Detailed Proof for Theorem \ref{thm:online-crl-lp}}

First, we bound total number of model updates in Algorithm \ref{algo:online-crl-lp}. 

\begin{lemma}
\label{lem:umax}
The total number of updates under algorithm \ref{algo:online-crl-lp} is bounded by $U_{\max} = |S|^2 |A| m.$
\end{lemma}
\begin{proof}
Let fix a $(s, a)-$pair. Note that $n(s, a)$ is not decreasing and also it increases up to $|S|m H.$ And, since update of model happens at the beginning of each episode, then maximum number of updates due to a single $(s, a)$ happens at most $|S| m$ number of times. Thus, maximum number of updates due to all $(s, a)-$pairs is no larger than $|S|^2|A| m$
\end{proof}

\textbf{\emph{Proof of Lemma \ref{lem:admis}}:} At each episode with model update $k$ and for each $(s,a),$ by Hoeffding's inequality \cite{hoeffding}  we have
\begin{align*}
    |P(s'|s,a) - \hat{P}(s'|s,a)| \leq \sqrt{\frac{\log{(4/\delta_1)}}{2n(s, a)}}
\end{align*}
holds w.p. at least $1 - \delta_1/2.$

By empirical Brenstein's inequality \cite{empirical-bernstein} we have
\begin{align*}
    |P(s'|s, a) - \widehat{P}(s'|s, a)| \leq \sqrt{\frac{2 \widehat{P}(s'|s, a) (1 - \widehat{P}(s'|s, a) )}{n(s, a)} \log{\frac{4}{\delta_1}}} + \frac{2}{3n(s, a)} \log{\frac{4}{\delta_1}}
\end{align*}
w.p. at least $1- \delta_1/2.$

Combining above two inequalities and applying union bound, we get
\begin{align*}
    \mathbb{P}(|P(s'|s, a) - \widehat{P}(s'|s, a)| \leq \min \{ \sqrt{\frac{2 \widehat{P}(s'|s, a) (1 - \widehat{P}(s'|s, a) )}{n(s, a)} \log{\frac{4}{\delta_1}}} + \frac{2}{3n(s, a)} \log{\frac{4}{\delta_1}}, \sqrt{\frac{\log{4/\delta_1}}{2n(s, a)}} \} ) \geq 1 - \delta_1.
\end{align*}
Finally, we get the result by applying union bound over all model updates and next states. \hfill  $\Box $


Now, we start proving Lemma \ref{lem:no-of-bad-episodes}. But, first we provide some useful lemmas.

\begin{lemma}
\label{lem:nsa}
Total number of observations of $(s,a) \in X_{k, \kappa, \iota}$  with $\kappa \in [1, |S|-1]$ and $\iota > 0$ over all phases $k$ is at most $3|S\times A|mw_\iota \kappa$. $w_\iota = \min \{w_k(s,a): \iota_k(s,a) = \iota \}$.
\end{lemma}

\begin{proof}
Note that $w_{\iota + 1} = 2w_\iota$ for $\iota > 0$. Consider a phase $k
$ and a fixed $(s,a) \in X_{k, \kappa, \iota}$. Since we assumed $\iota_k(s,a) = \iota$, then $w_\iota \leq w_k(s,a) \leq 2w_\iota$. Similarly, from $\kappa_k(s,a) = \kappa$ we have $\frac{n_k(s,a)}{2mw_k(s,a)} \leq \kappa \leq \frac{n_k(s,a)}{mw_k(s,a)}$ which implies
\begin{align}
    \label{eq:nsa-kappa}
    mw_\iota \kappa \leq mw_k(s,a)\kappa \leq n_k(s,a) \leq 2mw_k(s,a)\kappa \leq 4mw_\iota \kappa.
\end{align}

Therefore, each $(s, a)$ in $\{(s,a)\in X_{k,\kappa,\iota}: k \in \mathop{\mathbb{N}} \}$ can only be observed $3mw_\iota\kappa$. Then, the total observations is at most $3|S\times A|mw_\iota\kappa$.
\end{proof}

\begin{lemma}
\label{lem:episode}
Number of episodes $E_{\kappa,\iota}$ in phases with $|X_{k,\kappa,\iota}| > \kappa$ is bounded for $\alpha \geq 3$ w.h.p.
\begin{align*}
    P(E_{\kappa,\iota} > \alpha N) \leq \exp{(-\frac{\beta w_\iota(\kappa + 1)N}{H})},
\end{align*}
where $N=|S\times A|m$ and $\beta = \frac{\alpha(3/\alpha - 1)^2}{7/3 - 1/\alpha}.$
\end{lemma}

\begin{proof}
Let $\nu_k := \sum_{h=0}^{H-1} \mathop{\mathbb{I}} \{(s_h, a_h) \in X_{k, \kappa, \iota} \}$ be number of observations of $(s,a)$ with $|X_{k, \kappa, \iota}| > \kappa.$ We have $k \in \{1, ..., E_{\kappa, \iota} \}.$

In these episodes $|X_{k, \kappa, \iota}| \geq \kappa+1$ and all $(s,a)$ in partition $(\kappa, \iota)$ have $w_k(s,a) \geq w_\iota$, then
\begin{align*}
    \mathop{\mathbb{E}} [\nu_k|\nu_1,..., \nu_{k-1}] \geq (\kappa+1)w_\iota.
\end{align*}

Also $\mathop{\mathbb{V}}[\nu_k|\nu_1,...,\nu_{k-1}] \leq \mathop{\mathbb{E}}[\nu_k|\nu_1,..., \nu_{k-1}]H$ since $\nu_k \in [0,H].$

Now, we define the continuation:
\[\nu_k^+ :=  \begin{cases} 
      \nu_k & i\leq E_{\kappa, \iota} \\
      w_\iota(\kappa+1) & \text{O.W.}
   \end{cases}
\]
and centralized auxiliary sequence
\begin{align*}
    \bar{\nu}_k := \frac{\nu_k^+ w_\iota (\kappa +1)}{\mathop{\mathbb{E}}[\nu_k^+|\nu_1^+,..., \nu_{k-1}^+]}.
\end{align*}

By construction
\begin{align*}
    \mathop{\mathbb{E}}[\bar{\nu}_k|\bar{\nu}_1,...,\bar{\nu}_{k-1}] = w_\iota(\kappa+1).
\end{align*}

According to lemma \ref{lem:nsa}, we have $E_{\kappa, \iota} > \alpha N$ if
\begin{align*}
    \sum^{\alpha N}_{k=1} \bar{\nu}_k \leq 3Nw_\iota \kappa \leq 3Nw_\iota(\kappa+1).
\end{align*}

Now, we define martingale below
\begin{align*}
    B_k := \mathop{\mathbb{E}} \left[ \sum^{\alpha N}_{j=1} \bar{\nu}_j|\bar{\nu}_1,..., \bar{\nu}_k \right] = \sum^{k}_{j=1} \bar{\nu}_j + \sum^{\alpha N}_{j = k+1} \mathop{\mathbb{E}} [\bar{\nu}_j|\bar{\nu}_1, ..., \bar{\nu}_i],
\end{align*}
which gives $B_0 = \alpha Nw_\iota (\kappa + 1)$ and $B_{\alpha N} = \sum^{\alpha N}_{k=1} \bar{\nu}_k.$ Now, since $\nu^+_k \in [0, H]$
\begin{align*}
    |B_{k+1} - B_k| = |\bar{\nu}_k - \mathop{\mathbb{E}} [\bar{\nu}_k|\bar{\nu}_1, ..., \bar{\nu}_{k-1}]| = \left| \frac{w_\iota(\kappa + 1)(\nu^+_k - \mathop{\mathbb{E}}[\nu^+_k|\bar{\nu}_1,..., \bar{\nu}_{k-1}])}{\mathop{\mathbb{E}}[\nu^+_k|\nu^+_1,..., \nu^+_{k-1}] } \right| \leq |\nu^+_k - \mathop{\mathbb{E}}[\nu^+_k|\bar{\nu}_1,..., \bar{\nu}_{k-1}] | \leq H.
\end{align*}

Using
\begin{align*}
    \sigma^2 := \sum^{\alpha N}_{k=1} \mathop{\mathbb{V}} [B_k - B_{k-1}|B_1 - B_0, ..., B_{k-1}-B_{k-2}] = \sum^{\alpha N}_{k=1} \mathop{\mathbb{V}} [\bar{\nu}_k|\bar{\nu}_1,...,\bar{\nu}_{k-1}] \leq \alpha NHw_\iota (\kappa + 1) = HB_0
\end{align*}

we can apply Theorem $22$ of \cite{lu} and obtain
\begin{align*}
    \mathop{\mathbb{P}}(E_{\kappa, \iota} > \alpha N) \leq \mathop{\mathbb{P}} \left( \sum^{\alpha N}_{k=1} \bar{\nu}_k \leq 3Nw_\iota(\kappa + 1) \right) = \mathop{\mathbb{P}} (B_{\alpha N} - B_0 \leq 3B_0/\alpha - B_0) = \leq \exp{(-\frac{(3/\alpha-1)^2B_0^2}{2\sigma^2 + H(1/3 - 1/\alpha)B_0})}
\end{align*}
for $\alpha \geq 3.$ By simplifying it we get
\begin{align*}
    \mathop{\mathbb{P}}(E_{\kappa,\iota} > \alpha N) \leq \exp{-\frac{\alpha(3/\alpha)^2}{7/3 - 1/\alpha} \frac{Nw_\iota (\kappa+1)}{H}}.
\end{align*}
\end{proof}

\textbf{\emph{Proof of Lemma \ref{lem:no-of-bad-episodes}}:} Since $w_k(s,a) \leq H$, we have that $\frac{w_k(s,a)}{w_{min}} < \frac{H}{w_{min}}$ and so $\iota_k(s,a) \leq H/w_{min} = 4H^2|S|/\epsilon$. In addition, $|X_{k,\kappa,\iota}| \leq |S\times A|$ for all $k, \kappa, \iota$ and so $|X_{k,\kappa,\iota}| > \kappa$ can only be true for $\kappa \leq |S|$. Hence, only $E_{max} = \log_2{\frac{H}{w_{min}}}\log_2{|S|}$ possible values for $(\kappa,\iota)$ exists that can have $|X_{k,\kappa,\iota}| > \kappa$. By union bound over all $(\kappa,\iota)$ and lemma \ref{lem:episode}, we get
\begin{align*}
\mathop{\mathbb{P}} (E \leq \alpha NE_{max}) &\geq \mathop{\mathbb{P}}(\max_{(\kappa,\iota)} \leq \alpha N) \geq 1-E_{max}\exp{(-\frac{\beta w_\iota (\kappa+1)N}{H})}\\
&\geq 1-E_{max}\exp{(-\frac{\beta w_{min}N}{H})} = 1-E_{max}\exp{(-\frac{\beta w_{min}m|S\times A|}{H})}\\
&= 1- E_{max}\exp{(-\frac{\beta \epsilon m |S\times A|}{4H^2|S|})}.
\end{align*}

Bounding the right hand-side by $1-\delta/2$ and solving for $m$ gives
\begin{align*}
    1- E_{max}\exp{(-\frac{\beta \epsilon m |S\times A|}{4H^2|S|})} \geq 1-\delta/2 \Leftrightarrow m \geq \frac{4H^2|S|}{|S\times A|\beta \epsilon} \ln{\frac{2E_{max}}{\delta}}.
\end{align*}

Hence, the condition
\begin{align*}
    m \geq \frac{4H^2}{\beta \epsilon}\ln{\frac{2E_{max}}{\delta}}
\end{align*}
is sufficient for desired result to hold. Plugging in $\alpha = 6$ and $\beta = \frac{\alpha(3/\alpha-1)^2}{7/3 - 1/\alpha}$ would obtain the statement to show. \hfill  $\Box $

Next, we need the following corollaries to prove Lemma \ref{lem:bounded-mismatch}.

\begin{corollary}
If we substitute the $\delta_P$ with $\delta_1$ in Lemma \ref{lem:p-phat-ptilde}, the result will pertain.
\label{cor:p-phat-ptilde}
\end{corollary}

\begin{corollary}
If we substitute the $\delta_P$ with $\delta_1$ in Lemma \ref{lem:bound_on_pv}, the result will pertain.
\label{cor:bound_on_pv}
\end{corollary}

\textbf{\emph{Proof of Lemma \ref{lem:bounded-mismatch}:}} We only prove the statement of value function since the proof procedure for cost is identical.

Before proceeding, in this lemma we reason about a sequence of CMDPs $M_d$ which have the same transition probabilities but different reward matrix $r^{(d)}$ and cost matrices $c^{(d)}.$ Here, we only present the definition of $r^{(d)},$ as definition of $c^{(d)}$ is identical to $r^{(d)}.$ For $d = 0,$ the reward matrix is the original reward function $r$ of $M$ ($r^{(0)} = r$.) The following reward matrices are then defined recursively as $r^{(2d+2)} = \max_h \sigma^{(d),2}_{h:H-1}$, where $\sigma^{(d),2}_{h:H-1}$ is local variance of the value function w.r.t. the rewards $r^{(d)}.$ Note that for every $d$ and $h=0,...,H-1$ and $s\in S,$ we have $r^{(d)}(s) \in [0,H^d].$ 

In addition, we will drop the notations $k$ and policy $\tilde{\pi}_k$ in the following lemmas, since the statements are for a fixed episode $k$ and all value functions, reward matrices and transition kernels are defined under policy $\tilde{\pi}_k.$

Now,
\begin{align*}
\Delta_d & :=  |V^{(d)}_{0}(s_0) - \tilde{V}^{(d)}_0(s_0)| = \vert \sum^{H-2}_{h=0} P^{h-1}(P - \tilde{P})\tilde{V}^{(d)}_{h+1}(s_0) \vert \\
&\leq \sum^{H-1}_{h=0}P^{h-1}|P - \tilde{P} \tilde{V}^{(d)}_{h+1}|(s_0) \\
&=\sum^{H-1}_{h=0}P^{h-1} \left( \sum_{s,a\in S\times A} \mathbb{I} \{s = \cdot, a \sim \tilde{\pi}(s, \cdot, h) \}|(P - \tilde{P})\tilde{V}^{(d)}_{h+1}| \right)(s_0)\\
&= \sum_{s,a\in S\times A} \sum^{H-1}_{h=0} P^{h-1} \left( \mathbb{I} \{s = \cdot, a=\tilde{\pi}(s, \cdot, h) \}|(P - \tilde{P})\tilde{V}^{(d)}_{h+1}| \right)(s_0)\\
&=\sum_{s,a\in S\times A} \sum^{H-1}_{h=0} P^{h-1} \left( \mathbb{I} \{s = \cdot, a=\tilde{\pi}(s, \cdot, h) \}|(P - \tilde{P})\tilde{V}^{(d)}_{h+1}(s)| \right)(s_0)
\end{align*}

The first equality follows from Lemma \ref{lem:v-v'_pi}, the second step from the fact that $V_{h+1}\geq 0$ and $P^{h-1}$ being non-expansive. In the third, we introduce an indicator function which does not change the value as we sum over all $(s,a)$ pairs. The fourth step relies on the linearity of $P$ operators. In the fifth step, we realize that $\mathbb{I}\{s=., a \sim \tilde{\pi}(s, \cdot, h)\}|(P - \tilde{P})\tilde{V}^{(d)}_{h+1}(\cdot)|$ is a function that takes nonzero values for input $s.$ We can therefore replace the argument of the second term with $s$ without changing the value. The term becomes constant and by linearity of $P,$ we can write
\begin{align*}
&|V^{(d)}_{0}(s_0) - \tilde{V}^{(d)}_{0}(s_0)| = \Delta_d \leq \sum_{s,a\in S\times A} \sum^{H-1}_{h=0} P^{h-1} \left( \mathbb{I} \{s = \cdot, a \sim \tilde{\pi}(s, \cdot, h) \}|(P - \tilde{P})\tilde{V}^{(d)}_{h+1}(s)| \right)(s_0)\\
&\leq \sum_{s,a \not \in X} \sum^{H-1}_{h=0} \norm{\tilde{V}^{(d)}_{h+1}}_{\infty} (P^{h-1}\mathbb{I}\{ s=\cdot, a \sim \tilde{\pi}(s, \cdot, h)\})(s_0)\\
&+ \sum_{s,a\in X}\sum^{H-1}_{h=0} |(P - \tilde{P})\tilde{V}^{(d)}_{h+1}(s)|(P^{h-1}\mathbb{I} \{s = \cdot, a \sim \tilde{\pi}(s, \cdot, h) \})(s_0)\\
&\leq \sum_{s,a \not \in X} \sum^{H-1}_{h=0} H^{d+1} (P^{h-1} \mathbb{I}\{ s = \cdot, a \sim \tilde{\pi}(s, \cdot, h)\})(s_0)\\
&+ \sum_{s,a\in X}\sum^{H-1}_{h=0} |(P - \tilde{P})\tilde{V}^{(d)}_{h+1}(s)|(P^{h-1} \mathbb{I} \{s = \cdot, a \sim \tilde{\pi}(s, \cdot, h) \})(s_0)\\
&\leq \sum_{s,a \not \in X} \sum^{H-1}_{h=0} H^{d+1} (P^{h-1} \mathbb{I}\{ s=\cdot, a \sim \tilde{\pi}(s, \cdot, h)\})(s_0)\\
&+ \sum_{s,a\in X}\sum^{H-1}_{h=0} ||S| c_1(s,a)H^{d+1} + c_2(s,a) \sqrt{|S|} \tilde{\sigma}^{(d)}_{h}(s,a)| (P^{h-1} \mathbb{I} \{ s=\cdot, a \sim \tilde{\pi}(s, \cdot, h) \})(s_0)\\
&\leq \sum_{s,a \not \in X} \sum^{H}_{h=0} H^{d+1} (P^{h-1} \mathbb{I}\{ s=\cdot, a \sim \tilde{\pi}(s, \cdot, h)\})(s_0)\\
&+ \sum_{s,a\in X}\sum^{H}_{h=0} | |S| c_1(s,a) H^{d+1}|(P^{h-1}\mathbb{I}\{ s=\cdot, a \sim \tilde{\pi}(s, \cdot, h)\})(s_0)\\
&+ \sum_{s,a\in X}\sum^{H-1}_{h=0} | \sqrt{|S|} c_2(s,a)  \tilde{\sigma}^{(d)}_{h}(s,a)|(P^{h-1}\mathbb{I}\{ s = \cdot, a \sim \tilde{\pi}(s, \cdot, h)\})(s_0)\\
&\leq \sum_{s,a \not \in X} H^{d+1}w(s,a) + \sum_{s,a \in X} |S| c_1(s,a) H^{d+1} w(s,a)\\
&+\sum_{s,a\in X} \sqrt{|S|}c_2(s,a) \sum^{H-1}_{h=0}\tilde{\sigma}^{(d)}_{h}(s,a)(P^{h-1}\mathbb{I}\{ s = \cdot, a \sim \tilde{\pi}(s, \cdot, h)\})(s_0)\\
& \leq w_{min}|S|H^{d+1} + \sum_{s,a \in X} |S| c_1(s,a) H^{d+1} w(s,a) +\sum_{s,a\in X} \sqrt{|S|}c_2(s,a) \sum^{H-1}_{h=0}\tilde{\sigma}^{(d)}_{h}(s,a)(P^{h-1}\mathbb{I}\{ s = \cdot, a \sim \tilde{\pi}(s, \cdot, h)\})(s_0)\\
& = \frac{\epsilon}{4}H^d + \sum_{s,a \in X} |S| c_1(s,a) H^{d+1} w(s,a) +\sum_{s,a\in X} \sqrt{|S|}c_2(s,a) \sum^{H-1}_{h=0}\tilde{\sigma}^{(d)}_{h}(s,a)(P^{h-1}\mathbb{I}\{ s = \cdot, a \sim \tilde{\pi}(s, \cdot, h)\})(s_0)
\end{align*}

In the second inequality, we split the sum over all $(s,a)$ pairs and used the fact that $P$ and $\tilde{P}$ are non-expansive. The next step follows from $\norm{V^{(d)}_{h+1}}_{\infty}\leq\norm{V^{(d)}_{0}}_{\infty}\leq H^{d+1}.$ We then apply Lemma \ref{lem:bound_on_pv} and subsequently use that all terms are nonnegative and the definition of $w(s,a).$ Eventually, the last two lines come from the fact that $w(s,a)\leq w_{min}$ for all $(s,a)$ not in the active set. Besides, please note that we are analyzing under the given policy $\tilde{\pi},$ which implies that there are only $|S|$ nonzero $w$ in non-active set.

Using the assumption that $M \in \mathcal{M}$ and $\tilde{M} \in \mathcal{M}$ from the fact that ELP chooses the optimistic CMDP in $\mathcal{M},$ we can apply Corollary \ref{cor:p-phat-ptilde} and get that
\begin{align*}
    c_1(s,a) = 2\sqrt{2} \Bigl( \frac{\log{4/\delta_1}}{n(s, a)} \Bigr)^{3/4} + 3\sqrt{2} \frac{\log{4/\delta_1}}{n(s, a)}
    ~~~\text{and}~~~
    c_2(s,a) = \sqrt{\frac{8}{n(s,a)}\log{4/\delta_1}}.
\end{align*}

Plugging definitions above we have
\begin{align*}
    \Delta_d &\leq \frac{\epsilon}{4} H^d +  2\sqrt{2}|S| H^{d+1}\log{4/\delta_1}^{3/4}\sum_{s,a \in X} \frac{w(s, a)}{n(s, a)^{3/4}} + 3\sqrt{2}|S| H^{d+1}\log{4/\delta_1} \sum_{s,a \in X} \frac{w(s, a)}{n(s, a)}\\
    & + \sqrt{8 |S| \log{4/\delta_1}} \sum_{s,a\in X} \frac{1}{\sqrt{n(s, a)}} \sum^{H-1}_{h=0}\tilde{\sigma}^{(d)}_{h}(s,a)(P^{h-1}\mathbb{I}\{ s = \cdot, a \sim \tilde{\pi}(s, \cdot, h)\})(s_0)
\end{align*}

Hence, we bound
\begin{align*}
    \Delta_d \leq \frac{\epsilon}{4}H^d + U_d(s_0) + Y_d(s_0) + Z_d(s_0)
\end{align*}
as a sum of three terms which we will consider individually in the following. The first term is
\begin{align*}
    U_d(s_0) &= 2\sqrt{2}|S| H^{d+1}\log{4/\delta_1}^{3/4}\sum_{s,a \in X} \frac{w(s, a)}{n(s, a)^{3/4}}\\
    & \leq 2\sqrt{2}|S| H^{d+5/4}\log{4/\delta_1}^{3/4} \sum_{\kappa, \iota \in \mathcal{K} \times \mathcal{I}} \sum_{s, a \in X_{\kappa, \iota}} \Bigl(\frac{w(s, a)}{n(s, a)}\Bigr)^{3/4}\\
    & \leq 2\sqrt{2}|S| H^{d+5/4}\log{4/\delta_1}^{3/4} \sum_{\kappa, \iota \in \mathcal{K} \times \mathcal{I}} \Bigl(\frac{|X_{\kappa, \iota}|}{m\kappa}\Bigr)^{3/4}\\
    & \leq 2\sqrt{2}|S| H^{d+5/4}\log{4/\delta_1}^{3/4} \sum_{\kappa, \iota \in \mathcal{K} \times \mathcal{I}} \Bigl(\frac{1}{m}\Bigr)^{3/4}\\
    & \leq 2\sqrt{2}|S| H^{d+5/4}\log{4/\delta_1}^{3/4} \mathcal{K} \times \mathcal{I} \Bigl(\frac{1}{m}\Bigr)^{3/4}.
\end{align*}
In the second line, we used Cauchy-Scharwz. Next, we used the fact that for $s, a \in X_{\kappa, \iota},$ we have $n(s, a) \geq m w(s, a) \kappa,$ refer to equation \eqref{eq:nsa-kappa}. Finally, we applied the assumption of $|X_{\kappa, \iota}|\leq \kappa.$ Please note that $\mathcal{K\times\mathcal{I}}$ is the set of all possible $(\kappa,\iota)$ pairs.

The next term is
\begin{align*}
    Y_d(s_0) = 3\sqrt{2}|S| H^{d+1}\log{4/\delta_1} \sum_{s,a \in X} \frac{w(s, a)}{n(s, a)} \leq 3\sqrt{2}|S| H^{d+1}\log{4/\delta_1} \sum_{\kappa, \iota} \frac{|X_{\kappa, \iota}|}{m \kappa} \leq \frac{3\sqrt{2}|S| H^{d+1}\log{4/\delta_1} |\mathcal{K\times\mathcal{I}}|}{m}
\end{align*}
which we used $n(s, a) \geq m w(s, a) \kappa$ again.

The last term is
\begin{align*}
    &Z_d(s_0) = \sqrt{8 |S| \log{4/\delta_1}} \sum_{s,a\in X} \frac{1}{\sqrt{n(s, a)}} \sum^{H-1}_{h=0}\tilde{\sigma}^{(d)}_{h}(s,a)(P^{h-1}\mathbb{I}\{ s = \cdot, a \sim \tilde{\pi}(s, \cdot, h)\})(s_0) \leq \sqrt{8 |S| \log{4/\delta_1}} \\
    & \times \sum_{s,a\in X} \frac{1}{\sqrt{n(s, a)}} \sqrt{\sum^{H-1}_{h=0} P^{h-1}\mathbb{I}\{ s = \cdot, a \sim \tilde{\pi}(s, \cdot, h)\} (s_0) } \sqrt{\sum^{H-1}_{h=0}\tilde{\sigma}^{(d)^2}_{h}(s,a) P^{h-1}\mathbb{I}\{ s = \cdot, a \sim \tilde{\pi}(s, \cdot, h)\} (s_0)}\\
    & = \sqrt{8 |S| \log{4/\delta_1}} \sum_{s,a\in X} \sqrt{\frac{w(s, a)}{n(s, a)} \sum^{H-1}_{h=0}\tilde{\sigma}^{(d)^2}_{h}(s,a) P^{h-1}\mathbb{I}\{ s = \cdot, a \sim \tilde{\pi}(s, \cdot, h)\} (s_0)}\\
    & = \sqrt{8 |S| \log{4/\delta_1}} \sum_{\kappa, \iota} \sum_{s,a\in X_{\kappa, \iota}} \sqrt{\frac{w(s, a)}{n(s, a)} \sum^{H-1}_{h=0}\tilde{\sigma}^{(d)^2}_{h}(s,a) P^{h-1}\mathbb{I}\{ s = \cdot, a \sim \tilde{\pi}(s, \cdot, h)\} (s_0)}\\
    & \leq \sqrt{8 |S| \log{4/\delta_1}} \sum_{\kappa, \iota} \sqrt{ |X_{\kappa, \iota}| \sum_{s,a \in X_{\kappa, \iota}} \frac{w(s, a)}{n(s, a)} \sum^{H-1}_{h=0}\tilde{\sigma}^{(d)^2}_{h}(s,a) P^{h-1}\mathbb{I}\{ s = \cdot, a \sim \tilde{\pi}(s, \cdot, h)\} (s_0)}\\
    & \leq \sqrt{8 |S| \log{4/\delta_1}} \sum_{\kappa, \iota} \sqrt{\frac{1}{m} \sum_{s,a \in X_{\kappa, \iota}} \sum^{H-1}_{h=0}\tilde{\sigma}^{(d)^2}_{h}(s,a) P^{h-1}\mathbb{I}\{ s = \cdot, a \sim \tilde{\pi}(s, \cdot, h)\} (s_0)}\\
    & \leq \sqrt{\frac{8 |S| \log{4/\delta_1} |\mathcal{K} \times \mathcal{I}|}{m} \sum_{s,a \in X} \sum^{H-1}_{h=0}\tilde{\sigma}^{(d)^2}_{h}(s,a) P^{h-1}\mathbb{I}\{ s = \cdot, a \sim \tilde{\pi}(s, \cdot, h)\} (s_0)}\\
    & \leq \sqrt{\frac{8 |S| \log{4/\delta_1} |\mathcal{K} \times \mathcal{I}|}{m} \sum_{s,a \in S \times A} \sum^{H-1}_{h=0}\tilde{\sigma}^{(d)^2}_{h}(s,a) P^{h-1}\mathbb{I}\{ s = \cdot, a \sim \tilde{\pi}(s, \cdot, h)\} (s_0)}\\
    & = \sqrt{\frac{8 |S| \log{4/\delta_1} |\mathcal{K} \times \mathcal{I}|}{m} \sum^{H-1}_{h=0} P^{h-1} \tilde{\sigma}^{(d)^2}_{h} (s_0)}\\
    & \leq \sqrt{\frac{8 |S| H^{2d+3} \log{4/\delta_1} |\mathcal{K} \times \mathcal{I}|}{m}}.
\end{align*}
In the second line, we applied Cauchy-Scharwz inequality. Then, we used the definition of $w(s, a)$ to get to third step. Next, we split the sum and applied Cauchy-Scharwz again to obtain fifth step. Furthermore, we applied the assumption of $|X_{\kappa, \iota}| \leq \kappa$ to get sixth step. Next, we applied Cauchy-Scharwz inequality to obtain seventh step. And, the final step follows from the facts that $P^{h-1}$ is non-expansive and $\norm{\tilde{\sigma}^{(d)}_h}_{\infty} \leq H^{2d+2}.$ Thus, we have
\begin{align}
    \label{eq:first-bound}
    Z_d(s_0) \leq \sqrt{\frac{8 |S| H^{2d+3} \log{4/\delta_1} |\mathcal{K} \times \mathcal{I}|}{m}}.
\end{align}
However, we can improve this bound as follows
\begin{align*}
    Z_d(s_0) & \leq \sqrt{\frac{8 |S| \log{4/\delta_1} |\mathcal{K} \times \mathcal{I}|}{m} \sum^{H-1}_{h=0} P^{h-1} \tilde{\sigma}^{(d)^2}_{h} (s_0)}\\
    & = \sqrt{\frac{8 |S| \log{4/\delta_1} |\mathcal{K} \times \mathcal{I}|}{m} \sum^{H-1}_{h=0} P^{h-1} \tilde{\sigma}^{(d)^2}_{h} (s_0) - \tilde{P}^{h-1} \tilde{\sigma}^{(d)^2}_{h} (s_0) + \tilde{P}^{h-1} \tilde{\sigma}^{(d)^2}_{h} (s_0)}\\
    & \leq \sqrt{\frac{8 |S| \log{4/\delta_1} |\mathcal{K} \times \mathcal{I}|}{m} \Bigl( H^{2d+2} + \sum^{H-1}_{h=0} P^{h-1} r^{(2d+2)} (s_0) - \tilde{P}^{h-1} r^{(2d+2)} (s_0) \Bigr) }\\
    & = \sqrt{\frac{8 |S| \log{4/\delta_1} |\mathcal{K} \times \mathcal{I}|}{m} \Bigl( H^{2d+2} + V^{(2d+2)}_0(s_0) - \tilde{V}^{(2d+2)}_0(s_0) \Bigr) }\\
    & = \sqrt{\frac{8 |S| \log{4/\delta_1} |\mathcal{K} \times \mathcal{I}|}{m} ( H^{2d+2} + \Delta_{2d+2} ) }\\
    & \leq \sqrt{\frac{8 |S| \log{4/\delta_1} |\mathcal{K} \times \mathcal{I}|}{m}  H^{2d+2}} + \sqrt{\frac{8 |S| \log{4/\delta_1} |\mathcal{K} \times \mathcal{I}|}{m}  \Delta_{2d+2}}.
\end{align*}
In the third step, we used Lemma \ref{lem:variance-bound} and definition of $r^{(2d+2)}.$

Now, if we put all the pieces together, we have
\begin{align*}
    \Delta_d &\leq \frac{\epsilon}{4}H^d + 2\sqrt{2}|S| H^{d+5/4}\log{4/\delta_1}^{3/4} \mathcal{K} \times \mathcal{I} \Bigl(\frac{1}{m}\Bigr)^{3/4} + \frac{3\sqrt{2}|S| H^{d+1}\log{4/\delta_1} |\mathcal{K\times\mathcal{I}}|}{m}\\
    & + \sqrt{\frac{8 |S| \log{4/\delta_1} |\mathcal{K} \times \mathcal{I}|}{m}  H^{2d+2}} + \sqrt{\frac{8 |S| \log{4/\delta_1} |\mathcal{K} \times \mathcal{I}|}{m}  \Delta_{2d+2}}.
\end{align*}
If we choose $m$ sufficiently large which will be shown later, then it is straightforward to show that $U_d(s_0) \leq Z_d(s_0)$ and $Y_d(s_0) \leq Z_d(s_0).$ Hence, if we expand the above inequality up to depth $\beta = \ceil{\frac{\log{H}}{2 \log{2}}}$ with $\mathcal{D} = \{0, 2, 6, 14, \dots, \beta \},$ we get
\begin{align*}
    \Delta_0 &\leq \sum_{d\in \mathcal{D} \backslash {\beta}} \Bigl( \frac{8 |S| \log{4/\delta_1} |\mathcal{K} \times \mathcal{I}|}{m} \Bigr)^{\frac{d}{d+2}} \Bigl[\frac{\epsilon}{4} H^d + 3\sqrt{\frac{8|S| \log{4/\delta_1} |\mathcal{K} \times \mathcal{I}| H^{2d+2}}{m}} \Bigr]^{\frac{2}{d+2}}\\
    & + \Bigl( \frac{8 |S| \log{4/\delta_1} |\mathcal{K} \times \mathcal{I}|}{m} \Bigr)^{\frac{\beta}{\beta+2}} \Bigl[\frac{\epsilon}{4} H^{\beta} + 3\sqrt{\frac{8|S| \log{4/\delta_1} |\mathcal{K} \times \mathcal{I}| H^{2\beta+2}}{m}} \Bigr]^{\frac{2}{\beta+2}}.
\end{align*}
Here, we used inequality \eqref{eq:first-bound} to bound $Z_{\beta}(s_0).$ Finally, the proof completes if we let
\begin{align*}
    m = 1280 \frac{|S| H^2}{\epsilon^2} (\log_2 \log_2 H)^2 \log_2^2 \Bigl( \frac{8|S|^2 H^2}{\epsilon} \Bigr) \log{\frac{6}{\delta_1}}.
\end{align*}

\hfill $\Box.$

\paragraph{Proof of Theorem \ref{thm:online-crl-lp}:} By Lemma \ref{lem:no-of-bad-episodes}, we know that number of episodes where $|X_{\kappa, \iota}| > \kappa$ for some $\kappa, \iota$ is bounded by $6E_{\max}|S||A|m$ with probability at least $1 - \frac{\delta}{2(N+1)}.$ For all other episodes, we have by Lemma \ref{lem:bounded-mismatch} that for any $i \in \{ 1, \dots, N \}$
\begin{align}
\label{eq:thm-value}
    |\tilde{V}^{\tilde{\pi}_k}_0(s_0) - V^{\tilde{\pi}_k}_0(s_0)| \leq \epsilon, ~~~ |\tilde{C}^{\tilde{\pi}_k}_{i, 0}(s_0) - C^{\tilde{\pi}_k}_0(s_0)| \leq \epsilon.
\end{align}

Using Lemma \ref{lem:admis}, we get that $M \in \mathcal{M}_k$ for any episode $k$ w.p. at least $1 - \frac{\delta}{2(N+1)}.$ Further, we know that ELP outputs the policy $\tilde{\pi}_k$ such that
\begin{align}
\label{eq:elp-v}
    \tilde{V}^{\tilde{\pi}_k}_0(s_0) \geq V^{\pi^*}_0(s_0), ~~~ \tilde{C}^{\tilde{\pi}_k}_{i, 0}(s_0) \leq \bar{C}_i~~  i \in \{ 1, \dots, N \}
\end{align}
w.p. at least $1 - \frac{\delta}{2(N+1)}.$ Combining the inequalities \eqref{eq:thm-value} with inequalities \eqref{eq:elp-v}, we get that for all episodes with $|X_{\kappa, \iota}| \leq \kappa$ for all $\kappa, \iota$
\begin{align*}
    V^{\tilde{\pi}_k}_0(s_0) \geq V^{\pi^*}_0(s_0) - \epsilon
\end{align*}
w.p. at least $1 - \frac{\delta}{2(N+1)}$ and for any $i$, $ C^{\tilde{\pi}_k}_{i, 0}(s_0) \leq \bar{C}_i + \epsilon$ w.p. at least $1 - \frac{\delta}{2(N+1)}.$ Applying the union bound we get the desired result, if $m$ satisfies
\begin{align*}
    & m \geq 1280 \frac{|S|H^2}{\epsilon^2} (\log_2 \log_2 H)^2 \log_2^2 \Bigl( \frac{8 H^2 |S|^2}{\epsilon} \Bigr) \log{\frac{4}{\delta_1}} ~~ \text{and}\\
    & m \geq \frac{6 H^2}{\epsilon} \log{\frac{2(N+1) E_{\max}}{\delta}}.
\end{align*}

From the definitions, we get
\begin{align*}
    \log{\frac{4}{\delta_1}} = \log{\frac{4(N+1)|S|U_{\max}}{\delta}} = \log{\frac{4(N+1) |S|^2 |A| m }{\delta}}.
\end{align*}
Thus,
\begin{align*}
    m \geq 1280 \frac{|S|H^2}{\epsilon^2} (\log_2 \log_2 H)^2 \log_2^2 \Bigl( \frac{8 H^2 |S|^2}{\epsilon} \Bigr) \log{\frac{4(N+1) |S|^2 |A| m }{\delta}}.
\end{align*}
It is well-known fact that for any constant $B > 0, \nu \geq 2B \ln{B}$ implies $\nu \geq B \ln{\nu}.$ Using this, we can set
\begin{align*}
    m &\geq 2560 \frac{|S|H^2}{\epsilon^2} (\log_2 \log_2 H)^2 \log_2^2 \Bigl( \frac{8 H^2 |S|^2}{\epsilon} \Bigr)\\
    & \times \log{\Bigl( \frac{2048(N+1)|S|^3|A| H^2}{\epsilon^2 \delta} (\log_2 \log_2 H)^2 \log_2^2 \Bigl( \frac{8 H^2 |S|^2}{\epsilon} \Bigr) \Bigr)}.
\end{align*}

On the other hand,
\begin{align*}
    E_{\max} = \log_2 |S| \log_2 \frac{4|S| H^2}{\epsilon} \leq \log_2^2 \frac{4|S| H^2}{\epsilon}
\end{align*}
and
\begin{align*}
    \log{\frac{2(N+1)E_{\max}}{\delta}} &= \log{\frac{2 (N+1)\log_2|S| \log_2 (4|S| H^2 /\epsilon)}{\delta}} \leq \log{\frac{2(N+1) \log_2^2 (4|S| H^2/\epsilon)}{\delta}}\\
    & \leq \log{\frac{16 (N+1) |S|^4 |A| H^2}{\epsilon \delta}}.
\end{align*}

Setting
\begin{align}
\label{eq:m}
    m & = 2560 \frac{|S|H^2}{\epsilon^2} (\log_2 \log_2 H)^2 \log_2^2 \Bigl( \frac{8 H^2 |S|^2}{\epsilon} \Bigr)\\
    & \times \log{\Bigl( \frac{2048(N+1)|S|^4|A| H^2}{\epsilon^2 \delta} (\log_2 \log_2 H)^2 \log_2^2 \Bigl( \frac{8 H^2 |S|^2}{\epsilon} \Bigr) \Bigr)}. \nonumber
\end{align}
is therefore a valid choice for $m$ to ensure that with probability at least $1 - \frac{\delta}{(N+1)},$ there are at most
\begin{align*}
    6E_{\max}|S||A|m = & 15360 \frac{|S|^2 |A| H^2 }{\epsilon^2} (\log_2 \log_2 H)^2 \log_2^2 \Bigl( \frac{4|S| H^2}{\epsilon} \Bigr) \log_2^2 \Bigl( \frac{8 H^2 |S|^2}{\epsilon} \Bigr)\\
    & \times \log{\Bigl( \frac{2048(N+1)|S|^4|A| H^2}{\epsilon^2 \delta} (\log_2 \log_2 H)^2 \log_2^2 \Bigl( \frac{8 H^2 |S|^2}{\epsilon} \Bigr) \Bigr)}
\end{align*}
sub-optimal episodes. \hfill $\Box$

\end{document}